\documentclass[lettersize,journal]{IEEEtran}

\usepackage{url}
\usepackage[utf8]{inputenc} 
\usepackage[T1]{fontenc}    
\usepackage{bm}
\usepackage{cite}
\usepackage[colorlinks,
linkcolor=red,
anchorcolor=green,
citecolor=blue
]{hyperref}       

\usepackage{booktabs}       
\usepackage{amsfonts, amssymb}       
\usepackage{nicefrac}       
\usepackage{microtype}      
\usepackage{xcolor}         

\usepackage{amsmath}
\usepackage{algorithm}
\usepackage{algorithmic}
\usepackage{graphicx}
\usepackage{amsthm}
\usepackage{authblk} %

\newtheorem{theorem}{\bf Theorem}

\newtheorem{definition}{\bf Definition}
\newtheorem{proposition}{\bf Proposition}

\newcommand{\st}{\mathrm{s.t.}}

\usepackage{array}
\usepackage[caption=false,font=normalsize,labelfont=sf,textfont=sf]{subfig}
\usepackage{textcomp}
\usepackage{stfloats}
\usepackage{verbatim}
\def\BibTeX{{\rm B\kern-.05em{\sc i\kern-.025em b}\kern-.08em
    T\kern-.1667em\lower.7ex\hbox{E}\kern-.125emX}}
\usepackage{balance}
\usepackage{multirow}

\begin{document}

\title{A Learning Method with Gap-Aware Generation for Heterogeneous DAG Scheduling}


\author{Ruisong Zhou\IEEEauthorrefmark{1}, Haijun Zou\IEEEauthorrefmark{1}, Li Zhou, Chumin Sun, Zaiwen Wen 
\IEEEcompsocitemizethanks{This research was supported in part by National Key Research and Development Program of China under the grant number 2024YFA1012900, the National Natural Science Foundation of China under the grant numbers 12331010 and
12288101, and the Natural Science Foundation of Beijing, China under the grant number Z230002.
\textit{(Corresponding author: Zaiwen Wen.)}}
\IEEEcompsocitemizethanks{Ruisong Zhou is with the School of Mathematical Science, Peking University, Beijing 100871, China (email: ruisongzhou@stu.pku.edu.cn). }
\IEEEcompsocitemizethanks{Haijun Zou is with the State Key Laboratory of Mathematical Sciences, Institute of
Computational Mathematics and Scientific/Engineering Computing, Academy
of Mathematics and Systems Science, Chinese Academy of Sciences, Beijing
100190, China (email: zouhaijun24@mails.ucas.ac.cn).}
\IEEEcompsocitemizethanks{Li Zhou and Chumin Sun are with the Theory Lab, Central Research Institute, 2012 Labs, Huawei Technologies Co., Ltd, China (email: zhouli107@huawei.com; sunchumin@huawei.com).}
\IEEEcompsocitemizethanks{Zaiwen Wen is with the Beijing International Center for Mathematical Research, Peking University, Beijing 100871, China (email: wenzw@pku.edu.cn).}
\IEEEcompsocitemizethanks{\IEEEauthorrefmark{1}These authors contributed equally to this work.}
}

\maketitle

\begin{abstract}
Efficient scheduling of directed acyclic graphs (DAGs) is a core problem in large-scale data-intensive computing systems, where query plans, data-processing workloads, and computation graphs consist of dependent tasks competing for limited heterogeneous resource pools. In practice, achieving high-performance execution requires schedulers to adapt across environments with varying resource pools and task types, while generating schedules under tight runtime budgets. We propose WeCAN, an end-to-end reinforcement learning framework for heterogeneous DAG scheduling that addresses task-pool compatibility coefficients and generation-induced optimality gaps. It adopts a two-stage single-pass design: a single forward pass produces task-pool scores and global parameters, followed by a generation map that constructs schedules without repeated network calls. Its weighted cross-attention encoder models task-pool interactions gated by compatibility coefficients, and is size-agnostic to environment fluctuations. Moreover, widely used list-scheduling maps can incur generation-induced optimality gaps from restricted reachability. We introduce an order-space analysis that characterizes the reachable set of generation maps via feasible schedule orders, explains the mechanism behind generation-induced gaps, and yields sufficient conditions for gap elimination. Guided by these conditions, we design a skip-extended realization with an analytically parameterized decreasing skip rule, which enlarges the reachable order set while preserving single-pass efficiency. Experiments on real-world TPC-H query DAGs, resource-intensive workload datasets, and ML-compiler computation graphs demonstrate improved makespan over strong baselines, with inference time comparable to classical heuristics and faster than multi-round neural schedulers.

\end{abstract}


\section{Introduction}
In data-processing clusters \cite{mao2019acm}, heterogeneous cloud platforms \cite{Lin24TSC}, database query engines, and ML compilers \cite{jeon2023iclr}, many execution workloads are naturally represented as directed acyclic graphs (DAGs). This abstraction covers query execution plans, data-processing workloads, and computation graphs. In this abstraction, nodes represent computational operators, kernels, or job stages, edges specify data or control dependencies, and the scheduler determines both when each task starts and which resource pool executes it. The objective is typically to minimize the total completion time (makespan) subject to precedence and resource constraints, which is essential for high-performance execution in practical systems. This DAG abstraction covers a broad class of data and computing workloads, yet DAG scheduling is NP-hard even in homogeneous settings \cite{nphard}. In heterogeneous environments, tasks must further be assigned to suitable resource pools with different capacities and task-pool compatibility profiles. This heterogeneity adds significant complexity to scheduling.

Because scheduling decisions in data-processing clusters, query engines, compilers, and cluster managers must be made repeatedly under tight runtime budgets, traditional approaches often employ heuristics, such as list scheduling \cite{graham1969bounds}, which assigns tasks to pools iteratively based on priority scores. These scores are often calculated using computationally inexpensive metrics, such as the critical-path length or the number of remaining operations \cite{Haupt1989}. Variations include Tetris \cite{tetris} using dynamic scores that change during scheduling, or HEFT \cite{heft} based on inserting tasks into an existing timeline. However, the design of such heuristics often relies heavily on human expertise and struggles to fully utilize all available problem information effectively.

The application of machine learning (ML) to combinatorial optimization (CO) problems~\cite{kool2019am, kwon2020pomo, li2025lmask, zhou2024mvmoe,liu2023RGM,ma2024k-opt, hou2023TAM,ye2024glop} has gained significant attention in recent years since the seminal work by Hopfield \& Tank~\cite{mlco}. Within this line of work, a prevalent paradigm is to learn a constructive policy to generate solutions sequentially~\cite{kool2019am, kwon2020pomo, li2025lmask, zhou2024mvmoe,liu2023RGM}. This paradigm is particularly effective when the solution admits a natural sequential representation, so that constructing a solution reduces to generating an ordered selection, as in routing and assignment problems. 

Scheduling problems however differ fundamentally because the solution is defined as a schedule with concrete start times, which obstructs a natural one-to-one correspondence with a decision sequence. As a result, neural constructive schedulers instead learn dispatch priorities. Auto-regressive policies successively select tasks for dispatch~\cite{zhang2020nips,Song2023TwoStage}. But this requires multi-round network processing and limits inference speed. For rapid solution generation, a recent work~\cite{jeon2023iclr} advocates for a single-pass regime, where the whole priority list is generated in a one-shot manner. Under the learned priorities, schedules are generated by successively dispatching tasks to resource pools through a generation map.

List scheduling serves as the dominant generation map for neural constructive schedulers, yet it can introduce inherent optimality gaps in  DAG scheduling. This deficiency stems from the structural complexity of DAG scheduling. Different from the serial processing of linear chains in the canonical job-shop scheduling problem (JSSP), DAG scheduling involves cumulative resource constraints that permit parallel execution and intricate non-linear dependencies. Such a complex structure may render list scheduling unable to realize optimal schedules regardless of the given dispatch priorities. 

Heterogeneity further complicates the problem by introducing compatibility coefficients. Learning compatibility patterns while remaining agnostic to environment size poses significant challenges. One common approach is explicit injection \cite{Wu18AAPP, grin21cluster,jeon2023iclr}, where the model receives the full compatibility information as part of task features via one-hot task types or per-pool compatibility vectors. However, this design couples the network architecture to fixed input dimensions thereby limiting adaptability to fluctuating environments. Another approach is statistical compression~\cite{Zhou22MDM, Zhadan23acm, Wang2025ITJ}, which summarizes each task’s compatibility statistics across pools, such as the mean or variance. Nevertheless, this approach may lose some fine-grained information that matters for precise priority.

These observations motivate a need for constructive schedulers that are simultaneously compatibility expressive, environment size-agnostic, and provably free from generation-induced structural gaps, while remaining efficient in the single pass regime.  To this end, we develop a principled analysis of schedule generation maps for DAG scheduling and develop a reinforcement learning (RL) framework called WeCAN whose modeling and generation components are explicitly aligned with the requirements of heterogeneity and gap elimination. Our contributions are threefold.

(1) An order-space analysis framework revealing the generation-induced optimality gaps.
We introduce an order-space perspective to characterize the reachable set of existing schedule generation maps, and formalize a corresponding generation-induced optimality gap.
We further propose the order-covering optimal (OCO) criterion and prove that any generator satisfying OCO eliminates this structural gap, thereby explaining when and why common list-scheduling realizations may be inherently suboptimal.

(2) A single-pass, compatibility and gap aware RL scheduler.
We propose an end-to-end single-pass RL framework that jointly addresses heterogeneous compatibilities and the above structural gap.
On the modeling side, we design a size-agnostic weighted cross-attention network (WeCAN) to encode task--pool interactions with full compatibility information while remaining scalable to varying numbers of tasks and pools.
On the generation side, we introduce a skip-extended realization with a decreasing skip score, which enlarges the rollout space to fix the generation-induced gap, while preserving single-pass efficiency in inference.

(3) Empirical validation on representative data-intensive computing workloads.
Experiments on ML-compiler computation graphs and real-world TPC-H query DAGs demonstrate improved makespan over strong heuristic and learning-based baselines, with inference time comparable to classical heuristics and substantially faster than multi-round neural schedulers.
\subsection{Related Work}
Several RL based approaches have been proposed for scheduling problems~\cite{mao16acm, zhou2020nips, Wang21nips, Banerjee2020ICML, Park21ijpr, sun2021deepweave,Gag22nips, Islam2022TPDS, sun24IEEETC, Zhang2024iclrL2S, Debner2024PMLR, Li2025iclrlrho}. Among constructive schedulers, Zhang et al.~\cite{zhang2020nips} develop an approach to generate solutions by using a  graph neural network (GNN) to encode state and select tasks sequentially.  Jeon et al. \cite{jeon2023iclr} introduce an efficient approach with only once network passing, which uses the Gumbel-TopK trick to generate priorities that are subsequently converted into  feasible schedules by list scheduling. Beyond constructive generation, alternative approaches include learning policies to refine existing solutions \cite{zhou2020nips, Zhang2024iclrL2S} or simplify problem instances \cite{Wang21nips, sun24IEEETC, Li2025iclrlrho}. In particular, Zhou et al.~\cite{zhou2020nips} develop a method that learns a policy to iteratively refine an existing solution. Wang et al.~\cite{Wang21nips} introduce a bi-level optimization approach applying RL to add auxiliary edges and proposing heuristics on the modified graph.

Most previous works \cite{Wu18AAPP,Park21ijpr,Wang21nips,Zhou22MDM,jeon2023iclr,sun24IEEETC,Wang2025ITJ} adopt list scheduling as the generation map but do not include skip actions whenever ready tasks exist. 
These approaches thus inherit the optimality gap of list scheduling. 
Although skip actions have been considered in a few learning-based schedulers \cite{mao16acm, Grinsztajn2020SSCI, Banerjee2020ICML, grin21cluster, Islam2022TPDS, Zhadan23acm, Debner2024PMLR}, they are typically motivated by problem features other than generation-map reachability or list-scheduling optimality gaps. In particular, some works introduce skip in online or dynamic environments to wait for future task releases \cite{mao16acm, Grinsztajn2020SSCI, grin21cluster, Debner2024PMLR}, while others use it to cope with partially observed execution constraints \cite{Banerjee2020ICML, Islam2022TPDS} or as a coordination placeholder in multi-agent formulations \cite{Zhadan23acm}. Consequently, these methods do not analyze when and why adding skip can restore order coverage of the induced generation map, nor provide conditions under which skip eliminates the optimality gap. 
In contrast, we study the offline setting where all tasks are given upfront, and our order-space analysis shows that under cumulative resource constraints a list-scheduling map may fail to cover feasible schedule orders, leading to a nonzero optimality gap. Moreover, we show that closing this gap requires a carefully designed skip mechanism. Existing skip-based designs often rely on multi-round re-encoding or step-wise dynamic scoring, which conflicts with the single-pass inference regime. We therefore propose a decreasing skip rule that can be evaluated during rollout without additional network calls, and we prove that this design closes the gap while preserving single-pass efficiency.

Regarding task-pool compatibility, most heterogeneous DAG scheduling approaches incorporate compatibility coefficients directly into task features through various encoding choices~\cite{Wu18AAPP, grin21cluster,jeon2023iclr,Zhou22MDM,Zhadan23acm, Wang2025ITJ}. Heterogeneous GNNs (HGNNs) offer another potential avenue by extending graph learning to heterogeneous graphs with multiple node and edge types, and they have been adopted  for many other scheduling problems~\cite{Wang2022SchNet, Song2023TwoStage, Altundas2022IROS, Wang2023TNNLS, Zhang2024UAI}. For multi-robot scheduling, Wang et al.~\cite{Wang2022SchNet} propose ScheduleNet built on a heterogeneous attention mechanism. For the flexible JSSP, Song et al.~\cite{Song2023TwoStage} build an operation--machine heterogeneous graph where processing times are attached to edges, and use a staged embedding strategy to capture the coupled operation-selection and machine-assignment structure. Beyond these exemplars, a growing body of HGNN-based schedulers tailor heterogeneous graph constructions and message-passing modules to the particular modeling choices and constraints of different problem variants \cite{Altundas2022IROS, Wang2023TNNLS, Zhang2024UAI}. However, empirically adapting these HGNN methods to DAG scheduling often yields limited gains. Our WeCAN is complementary in that it targets efficient compatibility extraction without committing to a particular heterogeneous graph construction and it demonstrates superior performance on DAG scheduling. 
\subsection{Organization}
The remainder of this paper is structured as follows. Formulations of the heterogeneous DAG scheduling problem and neural constructive schedulers are provided in Section~\ref{sec:preliminary}. An order-analysis framework  for analyzing the inherent gap induced by generation maps is established in Section~\ref{sec:theory}. We present the proposed RL scheduler in Section~\ref{sec:generating-feasible-schedule}, including the skip-extended gap-free scheduling algorithm and the compatibility-aware WeCAN architecture.
Finally, numerical results on a variety of datasets are presented in Section~\ref{sec:experiments} to demonstrate the effectiveness and robustness of the proposed method.

\section{Preliminaries}\label{sec:preliminary}
\subsection{Heterogeneous Scheduling Problem}

The task scheduling problem seeks a schedule that minimizes the objective function among all schedules subject to given constraints. We focus on heterogeneous scheduling problems characterized by compatibility coefficients. A problem instance $\mathcal P=(V,E,\mathcal C,t,\rho,\lambda, K_{acc})$ is specified by a DAG $G=(V,E)$ and a set of resource pools $\mathcal C$.
Each node $v\in V$ represents a task with base processing time $t(v)>0$ and resource demand vector $\rho(v)$, while each pool $c\in\mathcal C$ has a capacity vector $\lambda(c)$.
Each edge $(v,w)\in E$ represents a dependency constraint: task $v$ must complete before $w$ starts.
We model heterogeneity via compatibility coefficients $K_{acc}(v,c)\ge 0$ \cite{heft}.
When $K_{acc}(v,c)>0$, the actual processing time of task $v$ on pool $c$ is defined as $
t_{\mathrm{act}}(v,c):=t(v)/K_{acc}(v,c)$, 
and $K_{acc}(v,c)=0$ indicates incompatibility.

A schedule is a map $x:V \to \mathbb R_{\geq0}\times \mathcal C$, which maps each task $v\in V$ to a pair $( s(v), c(v))$, representing its start time and assigned resource pool respectively. In this paper, the objective function is the makespan $f$, defined as the latest task completion time. For $\tau\ge 0$ and $c\in\mathcal C$, let $$
F(\tau,c):=\bigl\{v\in V\mid  c(v)=c,\, s(v)\leq\tau<t_{\mathrm{act}}(v, c(v))\bigr\} $$ denote the set of tasks running on pool $c$ at time $\tau$. The problem can be formulated as
\begin{equation}
\begin{aligned}
\min_{x=(s(\cdot), c(\cdot))}\quad&
f(x)=\max_{v\in V}\bigl[ s(v)+t_{\mathrm{act}}(v, c(v))\bigr],\\
\mathrm{s.t.}\qquad&
 s(v)+t_{\mathrm{act}}(v, c(v))\le  s(w),\quad \forall (v,w)\in E,\\
&
\sum_{v\in F(\tau,c)} \rho(v)\le \lambda(c),\quad \forall \tau\ge 0,\ \forall c\in\mathcal C,\\
&
K_{acc}(v, c(v))>0,\quad \forall v\in V,\\
&
 s(v)\ge 0,\quad \forall v\in V.
\end{aligned}
\end{equation}
The last constraint enforces nonnegative start times, while the first three constraints correspond to:
(1) dependency constraints, requiring each task to start after all its predecessors complete;
(2) resource constraints, ensuring that at any time $\tau$ the total demand of tasks running on pool $c$
does not exceed $\lambda(c)$; and (3) compatibility constraints, requiring $K_{acc}(v, c(v))>0$.  A schedule $x$ is called feasible if it satisfies all constraints above. An optimal solution is a feasible schedule whose makespan is no larger than that of any other feasible schedule.
 
In the homogeneous setting where $K_{acc}\equiv 1$ and $|\mathcal C|=1$, the third constraint is naturally satisfied. In contrast, heterogeneous environments with compatibility constraints pose additional challenges, requiring tasks to run on compatible resource pools and schedulers to adapt to diverse resource characteristics. These scheduling problems can be formulated as mixed integer linear programming (MILP) problems, as described in Appendix \ref{apx:theoritical results}.1. This formulation establishes a one-to-one correspondence between feasible schedules and MILP feasible solutions, enabling a unified framework for analyzing scheduling scenarios. Compared to the homogeneous case, these settings introduce additional constraints and significantly increase MILP size (as detailed in Appendix \ref{apx:dataset}.3), posing challenges for effective scheduling.

\subsection{MDP Formulation for Constructive Scheduling}
\label{subsec:prelim-L2schedule}

Most neural schedulers construct schedules  by making a sequence of dispatch decisions,
rather than directly outputting a complete schedule $x:V\to\mathbb R_{\ge 0}\times\mathcal C$.
For a given instance $\mathcal P$, we model this constructive procedure as a Markov decision process (MDP).

At decision step $t$, the state $s_t$ summarizes the current simulation time and the partial schedule built so far,
including (i) the set of unscheduled tasks, (ii) the dependency status (which predecessors have completed),
and (iii) the current status of each pool (available resources and running tasks).
The action space is 
\begin{equation*}
    \mathcal A=\{(v,c)\in V\times\mathcal{C}\mid \rho(v)\leq \lambda(c),K_{acc}(v,c)>0\},
\end{equation*}
where $(v,c)$ means assigning task $v$
to pool $c$.

A neural scheduler implements a stochastic policy, sampling $a_t\sim p_{\theta}(\cdot\mid s_t,\mathcal{P})$ at step $t$. 
A rollout (trajectory) is an action sequence $\omega=(a_1,\ldots,a_N)$ generated by the policy along the constructive simulation.
Its probability is
\begin{equation}
\label{eq:prob-mdp}
p_\theta(\omega) = \prod_{t=1}^{N} p_\theta(a_t \mid s_t, \mathcal P).
\end{equation}
Importantly, a rollout $\omega$ does not by itself specify concrete start times: it must be interpreted by a realization rule to produce a feasible schedule, from which the makespan is evaluated.  The overall reward is then defined as the negative of the makespan.

\subsection{From Constructive Policies to Feasible Schedules}
\label{subsec:prelim-constructive-to-schedule}

The constructive policy defines randomness on the decision side, i.e., action sequences, whereas
the scheduling objective function is defined on the solution side, i.e., feasible schedules with concrete start times.
Bridging these two requires a schedule realization rule (also called a \emph{generation map}) that translates constructive decisions into a feasible schedule. A standard choice in both heuristic and neural schedulers is list scheduling.
Starting from $t_{\mathrm{cur}}=0$, this scheme repeatedly:
(i) determines which task-pool actions are eligible at the current time (respecting precedence, capacity, and compatibility);
(ii) selects one eligible action according to the constructive decisions, starts the task at time $t_{\mathrm{cur}}$,
and updates the pool states; and
(iii) when no further actions are eligible, advances $t_{\mathrm{cur}}$ to the next completion event and releases completed tasks.

Once the generation scheme is fixed, the rollout randomness on the decision side induces a distribution over feasible schedules. Specifically,
the probability of obtaining a schedule is the total probability of all rollouts that are realized into that schedule. Therefore, learning a constructive policy amounts to optimizing the expected objective value under this induced schedule distribution. Different generation rules can lead to different induced distributions. In particular, the set of feasible schedules that can be realized with nonzero probability may differ. We will formalize this dependence in the later analysis.

\section{Order-Space Geometry and Optimality Gaps}\label{sec:theory}

In this section, we will introduce a framework to analyze the limitation of current list-scheduling based approaches and reveal the optimality gap inside them. This framework provides insight for revealing the underlying mechanisms for the suboptimality and guides us to find the best way to improve it.
\subsection{Schedule Order Space}
For any problem $\mathcal{\mathcal{P}}$, let $\mathcal{X}(\mathcal{P})$ denote the space of all feasible schedules. When no confusion arises, we denote $\mathcal{X}(\mathcal{P})$ as $\mathcal{X}$. Although the optimization objective resides in $\mathcal{X}$, the  probabilistic model is defined not directly over the schedule space itself, but over the action sequence space $\Omega_{\text{action}}$:
    \begin{equation*}
\begin{aligned}
    \Omega_{\text{action}}: = \{(a_1,...,a_n) \mid  a_i =(v_i,\bar{c}_i) \in \mathcal A; \,v_i\ne v_j,\, \forall\, i\ne j\}.
\end{aligned}
\end{equation*}
To facilitate a more compact representation for generating schedules, we investigate the structural properties of $\Omega_{\text{action}}$ in this subsection.

The action sequence  exhibits inherent redundancy since swapping two actions involving different pools yields the same schedule. An equivalence relation $\sim$ can be defined as follows:
\begin{equation*} (a_1,...,a_i,a_{i+1},...a_n) \sim (a_1,...,a_{i+1},a_i,...a_n) \text{ if } \bar{c}_i\ne \bar{c}_{i+1}.
\end{equation*}
The quotient space $\Omega_{\mathrm{action}}/\!\sim$ can be identified with an order representation defined as follows. A \emph{schedule order} $w=\{d(v),c(v)\}_{v\in V}$ consists of a pool-assignment map $c(\cdot):V\to\mathcal{C}$ and a rank map $d(\cdot):V\to\mathbb Z^+$. The pool-assignment map is required to be consistent with the action space, such that $(v,c(v))\in\mathcal{A}$ for all $v\in V$. For each pool $c\in \mathcal{C}$, let $n(c):=\#\{v\in V\mid c(v)=c\}$ denote the number of tasks assigned to that pool. The rank map is required to satisfy 
\begin{equation*}
 \{d(v)\mid c(v)=c\}=\{1,\dots,n(c)\},\quad \forall\, c\in\mathcal C.
\end{equation*}
Intuitively, $w=\{d(v),c(v)\}_{v\in V}$ specifies the pool allocation $c(v)$ for each task $v$ and the execution order $d(v)$ of $v$ within its assigned pool. The set of all such orders constitutes the \emph{schedule order space}, denoted by $\mathcal{B}$. The following proposition establishes the equivalence between the quotient space $\Omega_{\mathrm{action}}/\sim$ and the schedule order space $\mathcal{B}$. 
\begin{proposition}[Equivalence between action-quotient and order space]\label{prop:canical-representation-action-seq}
There exists a bijection $\Phi:\Omega_{\rm{action}}/\!\sim \,\to \mathcal B$.
In particular, each equivalence class in $\Omega_{\rm{action}}/\!\sim$ corresponds to a unique
schedule order $w=\{d(v),c(v)\}_{v\in V}\in\mathcal B$.
\end{proposition}

 \begin{proof}
Fix a total order of pools $\mathcal C=\{c_1,\dots,c_m\}$.
For any action sequence $\omega=(a_1,\dots,a_n)\in\Omega_{\mathrm{action}}$ with $a_i=(v_i,\bar{c}_i)$,
define its \emph{canonical form} $\mathrm{can}(\omega)$ as follows. For each $j=1,\dots,m$, extract from $\omega$ the subsequence of actions whose associated pool equals $c_j$,
keep their relative order, and concatenate these $m$ subsequences in the order
$c_1,\dots,c_m$.
Equivalently,
\[
\mathrm{can}(\omega)=\bigl(\ \omega|_{c_1}\ ;\ \omega|_{c_2}\ ;\ \cdots\ ;\ \omega|_{c_m}\ \bigr),
\]
where $\omega|_{c}$ denotes the subsequence of $\omega$ consisting of actions involving pool $c$.

(a) $\omega\sim \mathrm{can}(\omega)$.
Indeed, by repeatedly swapping adjacent actions belonging to different pools (which is allowed by $\sim$),
we can bubble all actions involving pool $c_1$ to the front without changing their internal order; then do the same for $c_2$, etc.
Thus, $\omega$ can be transformed into $\mathrm{can}(\omega)$ using only $\sim$-swaps.

(b) $\omega\sim\omega'$ iff $\mathrm{can}(\omega)=\mathrm{can}(\omega')$.
The ``if'' direction follows from Claim 1:
$\omega\sim\mathrm{can}(\omega)=\mathrm{can}(\omega')\sim\omega'$.
For the ``only if'' direction, given that a $\sim$-swap exchanges two adjacent actions involving different pools,
it does not change, for any fixed pool $c$, the relative order of actions involving pool $c$.
Therefore, each subsequence $\omega|_c$ is invariant along the entire $\sim$-class, implying the canonical form is invariant. Consequently, each equivalence class $[\omega]\in\Omega_{\mathrm{action}}/\!\sim$ contains exactly one canonical representative.

(c) Canonical forms are in bijection with $\mathcal B$.
Given a canonical form, within each block corresponding to pool $c$,
the position of a task $v$ inside that block defines $d(v)\in\{1,\dots,n(c)\}$ and the block identity defines $c(v)=c$,
producing an element $w=\{(d(v),c(v))\}_{v\in V}\in\mathcal B$.
Conversely, for any $w\in\mathcal B$, each pool $c$ has a unique ordered list of tasks with ranks $1,\dots,n(c)$. Concatenating these lists in the fixed pool order yields a unique canonical form.
These two constructions are inverses, giving a bijection between the set of canonical forms and $\mathcal B$.

Combining the above, $\Omega_{\mathrm{action}}/\!\sim$ (identified with canonical representatives)
is in bijection with $\mathcal B$.
\end{proof}

 Naturally, for a feasible schedule $x\in \mathcal X$, we can extract its task orders and pool allocations and obtain a schedule order in the schedule order space $\mathcal B$. This induces a map $T:\mathcal X\rightarrow \mathcal B$. However, not every schedule order in $\mathcal{B}$ necessarily arises from a feasible schedule via the map $T$, due to the presence of dependency and resource constraints.  We call the schedule orders in the image of $T$ the feasible schedule orders. 
\begin{definition}
    The feasible schedule order space $\mathcal B_f := T(\mathcal X)\subseteq \mathcal B$ is the image of $T$. Each element $w$ in $\mathcal B_f$ is called a feasible schedule order.
\end{definition}

The feasibility of a schedule order can be characterized by augmenting the precedence constraints with the within-pool ordering constraints induced by $w$.
Given $w=\{(d(v),c(v))\}_{v\in V}\in\mathcal B$, define
\begin{align}
&E_w^{\mathrm{pool}}:=\{(u,v)\in V\times V\mid\ c(u)=c(v),\ d(u)<d(v)\},\notag \\
&G_w:=(V,\,E\cup E_w^{\mathrm{pool}}).\label{eq:Gw}
\end{align}    

\begin{proposition}[Feasibility of schedule orders]
\label{pro:eq-fea}
A schedule order $w=\{(d(v),c(v))\}_{v\in V}\in\mathcal B$ belongs to $\mathcal{B}_f=T(\mathcal{X})$
if and only if
(i) $G_w$ is acyclic, and
(ii) $K_{acc}(v,c(v))>0$ and $\rho(v)\le \lambda(c(v))$ for all $v\in V$.
\end{proposition}

\begin{proof}
(\emph{Necessity}) If $w\in \mathcal B_f$, there exists a feasible schedule $x=(s(\cdot), c(\cdot))$ with $T(x)=w$.
Given that $E_w^{\text{pool}}$ has no directed cycle, every directed cycle in $G_w$ has at least an edge from $E$. For any edge $(u,v)\in E$, feasibility implies $ s(u)+t(u)/K_{acc}(u, c(u))\le s(v)$, and $ s(u)<  s(v)$. 
For any $(u,v)\in E_w^{\mathrm{pool}}$, $T(x)=w$ means $u$ is scheduled no later than $v$ on the same pool, which leads to $ s(u)\leq  s(v)$.
Consequently, $G_w$ cannot contain a directed cycle. Moreover, feasibility of $x$ implies $K_{acc}(v,c(v))>0$ and $\rho(v)\le\lambda(c(v))$ for all $v$.

(\emph{Sufficiency}) Assume (i) and (ii). Since $G_w$ is acyclic, fix any topological ordering $\pi$ of $G_w$. Define a schedule $x=( s, c)$ as follows.
Set $ c(v):=c(v)$ for all $v$.
Then define $ s(\cdot)$ recursively along order $\pi$:
for each $v$ processed in the order of $\pi$, let
\begin{equation*}
\begin{aligned}
     s(v):=\max\left\{{s}(u)+ t_{\text{act}}(u, c(u))\mid (u,v)\in E\cup E_w^{\mathrm{pool}}\right\}.
\end{aligned}    
\end{equation*}
Then all precedence constraints and within-pool order constraints are satisfied by construction.
Moreover, tasks on the same pool never overlap in time. Hence, at any $\tau$ each pool runs at most one task, and the resource constraint holds because $\rho(v)\le\lambda(c(v))$ for every task.
Compatibility holds by (ii). Thus $x$ is feasible and satisfies $T(x)=w$, and $w\in\mathcal B_f$.
\end{proof}

\subsection{Generation Maps}

To bridge the gap between decision and solution spaces, we require a realization rule that converts an action sequence into a schedule. In light of Proposition~\ref{prop:canical-representation-action-seq}, which identifies action sequences with canonical schedule orders up to redundancy, it suffices to formulate this realization primarily over schedule orders. Hence, we use the term \emph{generation map} to refer to any map $S: \mathcal{W} \to \mathcal{X}$ that constructs a fully specified schedule from an abstract order representation. The domain $\mathcal{W}$ serves as a generalized  priority space, which, depending on the context, may represent the raw action sequence space $\Omega_{\mathrm{action}}$, the schedule order space $\mathcal{B}$ or its feasible subset $\mathcal{B}_f$ etc.

A standard example is the list-scheduling map $S_{\text{list}}:\Omega_{\mathrm{action}}\rightarrow \mathcal X$, widely used in schedulers. Given an action sequence, $S_{\text{list}}$ follows the time-advance procedure described in Section~\ref{subsec:prelim-constructive-to-schedule} and schedules tasks according to the prescribed within-pool ordering.

We next introduce another generation map, the serial schedule generation scheme (SGS), denoted by  $S_{\text{SGS}}$, as specified in Algorithm~\ref{alg:optimal-section}.
Given a feasible schedule order $w\in\mathcal B_f$, the map $S_{\text{SGS}}$ constructs a schedule by combining
(i) the augmented precedence graph $G_w$ (cf.~\eqref{eq:Gw}) and its ready set $R$ (zero-indegree tasks), and
(ii) an event-based time-advance on each pool that schedules a selected ready task at its earliest feasible start time
under dependency and resource constraints. The next proposition shows that $S_{\text{SGS}}$ is well defined on $\mathcal B_f$.

\begin{algorithm}[tb]
\caption{Serial schedule generation scheme $S_{\text{SGS}}$}
\label{alg:optimal-section}
\begin{algorithmic}[1]
\STATE {\bfseries Input:} $(V,E,\mathcal C,\rho,t,\lambda,K_{\text{acc}})$ and a feasible order $w=\{(d(v),c(v))\}_{v\in V}\in\mathcal B_f$.
\STATE Let $t_{\text{act}}(v,c)=t(v)/K_{\text{acc}}(v,c)$ and build $G_w$ as \eqref{eq:Gw}.
\STATE Initialize $\deg(\cdot)$ in $G_w$, $R\gets\{v\mid\deg(v)=0\}$, $t_{\text{dep}}(v)\gets 0$, and for each $c\in\mathcal C$: $t_{\text{cur}}(c)\gets 0$, $\lambda_{\text{cur}}(c)\gets\lambda(c)$.

\WHILE{$R\neq\emptyset$}
    \STATE Pick any $v\in R$; set $c\gets c(v)$.
    \STATE Set $\tau\gets \max\{t_{\text{cur}}(c),\,t_{\text{dep}}(v)\}$.
    \STATE Advance time on pool $c$ until $v$ is feasible at $\tau$ (release completed tasks on $c$ and update $\lambda_{\text{cur}}(c)$); update $\tau$ accordingly.
    \STATE Start $v$ at $ s(v)\gets \tau$ on pool $c$; set $t_{\text{end}}(v)\gets  s(v)+t_{\text{act}}(v,c)$ and update the pool state $(t_{\text{cur}},\lambda_{\text{cur}})$.
    \STATE Remove $v$ from $R$; for each successor $u$ of $v$ in $G_w$: decrease $\deg(u)$ and add $u$ to $R$ if $\deg(u)=0$.
    \FOR {each $(v,u)\in E$} \STATE update $t_{\text{dep}}(u)\gets \max\{t_{\text{dep}}(u),\,t_{\text{end}}(v)\}$.
    \ENDFOR
\ENDWHILE
\STATE {\bfseries Output:} schedule $x=( s(\cdot), c(\cdot))\in\mathcal X$.
\end{algorithmic}
\end{algorithm}

\begin{proposition}[Well-definedness of $S_{\text{SGS}}$]
\label{pro:SGS-well-defined}
For any feasible schedule order $w=\{(d(v),c(v))\}_{v\in V}\in\mathcal B_f$, Algorithm~\ref{alg:optimal-section}
terminates after $|V|$ iterations and outputs a feasible schedule
$x=S_{\text{SGS}}(w)=( s(\cdot), c(\cdot))\in\mathcal X$ that preserves the order, i.e.,
$T(x)=w$.
\end{proposition}

\begin{proof}
Fix $w\in\mathcal B_f$. By Proposition~\ref{pro:eq-fea}, the augmented graph $G_w=(V,E\cup E_w^{\mathrm{pool}})$ is acyclic,
and the assignment satisfies $K_{\mathrm{acc}}(v,c(v))>0$ and $\rho(v)\le \lambda(c(v))$ for all $v\in V$.

\emph{a) Termination.}
Algorithm~\ref{alg:optimal-section} maintains the indegrees $\deg(\cdot)$ with respect to the remaining subgraph induced by the unscheduled tasks.
Since $G_w$ is acyclic, any nonempty induced subgraph has at least one node with indegree $0$. The ready set
$R=\{v:\deg(v)=0\}$ will be nonempty whenever tasks remain.
Since each iteration selects one $v\in R$ and permanently removes it (updating indegrees of its successors), exactly one task is scheduled per iteration.
The algorithm terminates after $|V|$ iterations.

\emph{b) Feasibility.}
Consider an iteration that schedules task $v$ on pool $c=c(v)$.
The algorithm sets $\tau\gets \max\{t_{\text{cur}}(c),t_{\text{dep}}(v)\}$ and then advances time on pool $c$ until $v$ becomes feasible at time $\tau$,
meaning that (i) all tasks on pool $c$ completed by time $\tau$ have been released and $\lambda_{\text{cur}}(c)$ equals the available capacity at $\tau$,
and (ii) $\rho(v)\le \lambda_{\text{cur}}(c)$ holds at $\tau$.
Moreover, by construction $\tau\ge t_{\text{dep}}(v)$, and $t_{\text{dep}}(v)$ is maintained as a lower bound on the completion times of all predecessors of $v$ in $E$,
since after scheduling any edge $(u,v)\in E$ the update
$t_{\text{dep}}(v)\gets \max\{t_{\text{dep}}(v),t_{\text{end}}(u)\}$ is performed.
Hence, starting $v$ at $ s(v)=\tau$ satisfies the dependency constraints.
The resource constraint holds by the pool-wise event update and the check $\rho(v)\le \lambda_{\text{cur}}(c)$ at the start time.
Compatibility holds since $K_{\mathrm{acc}}(v,c(v))>0$ for all $v$.
Therefore, the schedule lies in $\mathcal X$.

\emph{c) Order preservation.}
Pool allocations are preserved because the algorithm always schedules each task $v$ on its prescribed pool $c(v)$.
For within-pool order, note that $E_w^{\mathrm{pool}}$ encodes the rank order within each pool (as in \eqref{eq:Gw});
in particular, if $c(u)=c(v)$ and $d(u)<d(v)$, then there is a directed path from $u$ to $v$ in $G_w$.
Thus, $v$ cannot enter the ready set before all such predecessors have been scheduled, implying that tasks on each pool are scheduled in increasing $d(\cdot)$.
Consequently, extracting the within-pool order from the produced schedule yields exactly $w$, i.e., $T(x)=w$.
\end{proof}

\begin{figure*}[t]
    \centering
    \includegraphics[width=0.7\linewidth]{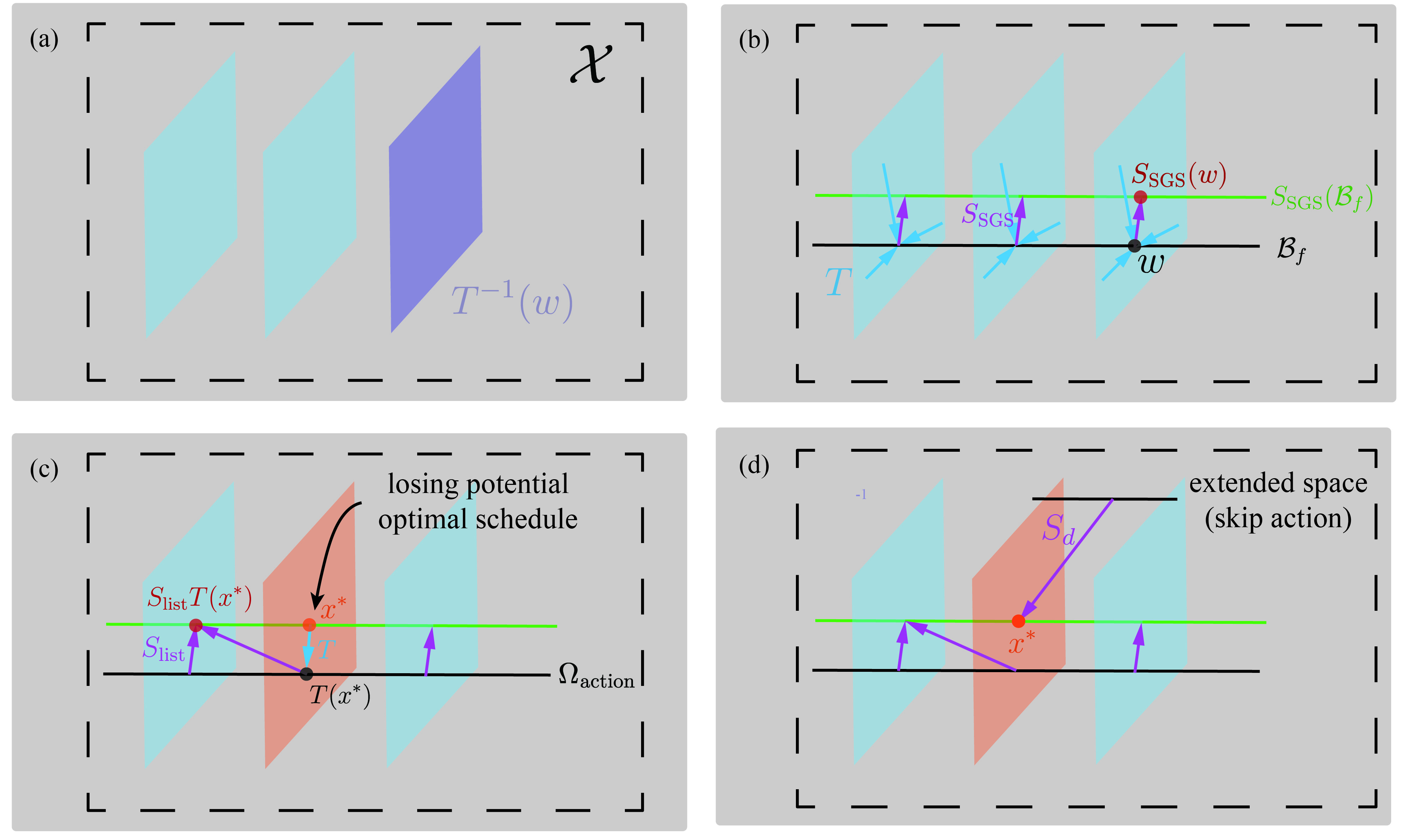}
    \caption{The illustration for spaces and maps. (a) The schedule space, divided into disjoint fibers for feasible schedule orders. (b) The serial schedule generation scheme $S_{\text{SGS}}$ embeds the feasible schedule order space $\mathcal B_f$ as a subset of the schedule space $\mathcal X$, which is ``orthogonal'' to each fiber $T^{-1}(w)$ as only intersect at one point $S_{\text{SGS}}(w)$. $S_{\text{SGS}}$ maps each schedule order to a best schedule (with minimal objective value) in the corresponding fiber. This relationship fix the optimality gaps. (c) The list scheduling map $S_{\text{list}}$ maps each schedule order to a best schedule in some fiber, but the fiber may not be the corresponding fiber. Therefore, the fiber containing the optimal solution may be missed since the optimal schedule order may be mapped to the schedule with other schedule orders. This mismatch leads to an optimality gap. (d) The skip action map $S_d$ is defined on a larger space $ \overline{\Omega}_{\text{action}}\supset \Omega_{\text{action}}$, where the extended part $S_d(\overline{\Omega}_{\text{action}}\backslash \Omega_{\text{action}})$ covers the missed part $S_{\text{SGS}}(\mathcal  B_f)\backslash S_{\text{list}}(\Omega_{\text{action}})$, making $TS_d$ a surjection to $\mathcal B_f$. The extended part closes the optimality gap.}
    \label{fig:illu-space}
\end{figure*}
Proposition~\ref{pro:SGS-well-defined} justifies viewing $S_{\text{SGS}}$ as a generation map from feasible orders to feasible schedules.

\subsection{Optimality Gap of List Scheduling}

A generation map $S$ does not necessarily realize all feasible schedules. Instead it induces a
restricted reachable set
\[
\mathcal X(\mathcal P;S)\;:=\;\{\,S(w)\mid w\in \mathcal W(\mathcal P)\,\}\ \subseteq\ \mathcal X(\mathcal P),
\]
where $\mathcal W(\mathcal P)$ is the valid domain of $S$ for instance $\mathcal P$. 
Accordingly, we define the \emph{optimality gap} of $S$ on $\mathcal P$ by
\begin{equation}
\label{eq:gap-def}
\delta(\mathcal P,S)\;:=\;\min_{x\in \mathcal X(\mathcal P;S)} f(x)\;-\;\min_{x\in \mathcal X(\mathcal P)} f(x)\ \ge\ 0 .
\end{equation}
When $\mathcal X(\mathcal P;S)\subsetneq \mathcal X(\mathcal P)$, this gap can be strictly positive. Notably, list scheduling $S_{\text{list}}$ can exhibit a nonzero optimality gap, i.e., there exists an instance $P_0$ such that $\delta(\mathcal P_0,S_{\text{list}})>0$.

Consider an instance $\mathcal P_0$ with $V=\{1,2,\dots,8\}$, $\mathcal C=\{c_1\}$ and $E=\{(1,4),(1,6),(4,7),(2,5),(5,8)\}$. Task processing times are $t(2)=1.1$, $t(3)=1.2$ and 
$t(i)=1$ for $i\neq 2,3$. Resource demands are $\rho(4)=2$ and $\rho(i)=1$ for $i\neq 4$.
The pool capacity is $\lambda(c_1)=3$ and $K_{\text{acc}}\equiv 1$. Let $x^\star=( s^\star, c^\star)$ be the feasible schedule defined by $ c^\star(v)\equiv c_1$ and $  s^\star(1)= s^\star(2)= s^\star(3)=0,\quad
 s^\star(4)=1.1,\quad  s^\star(5)=1.2,\quad
 s^\star(6)= s^\star(7)=2.1,\quad  s^\star(8)=2.2,$
which yields $f(x^\star)=3.2$. Now consider list scheduling $S_{\text{list}}$ as a map from $\mathcal B(\mathcal P_0)$ to $\mathcal X(\mathcal P_0)$.
For this specific instance, one can verify that for every action sequence $\omega\in \Omega_{\mathrm{action}}(\mathcal{P}_0)$,
list scheduling produces the same schedule $x_0=( s_0, c_0)$, where $ c_0(v)\equiv c_1$ and $ s_0(1)= s_0(2)= s_0(3)=0,\quad
 s_0(6)=1,\quad  s_0(5)=1.1,\quad
 s_0(4)=2,\quad  s_0(8)=2.1,\quad  s_0(7)=3$, resulting in $f(x_0)=4$ (see Fig.~\ref{fig:counterex}).
Therefore, $\mathcal X(\mathcal P_0;S_{\text{list}})=\{x_0\}$, and
\[
\delta(\mathcal P_0,S_{\text{list}})=f(x_0)-f(x^\star)=4-3.2=0.8>0,
\]
indicating the nonzero optimality gap of list scheduling on $\mathcal P_0$.

The existence of optimality gap admits a geometric interpretation. To see this, we consider the disjoint decomposition of the schedule space $\mathcal{X}(\mathcal{P})$ through the map $T$
\begin{equation}
\label{eq:decomposition}
    \mathcal X(\mathcal P)=\bigsqcup_{w\in \mathcal B_f(\mathcal P)} T^{-1}(w).
\end{equation}
Here each preimage $T^{-1}(w)$ is called a fiber, consisting of all schedules sharing the same feasible order $w$ (cf. Fig.~\ref{fig:illu-space}(a)). 
Since $S_{\text{list}}$ only explores its image $S_{\text{list}}(\Omega_{\mathrm{action}}(\mathcal P))\subsetneq \mathcal X(\mathcal P)$,
it may fail to intersect some fibers.
Equivalently, the composed map $T\circ S_{\text{list}}:\Omega_{\mathrm{action}}(\mathcal P)\to \mathcal B_f(\mathcal P)$
is not necessarily surjective. This surjectivity failure  arises from the non-injective nature of $T\circ S_{\text{list}}$. Specifically, 
different inputs $\omega_1\neq \omega_2$ may lead to schedules  that induce the same
schedule order, i.e., $T(S_{\text{list}}(\omega_1))=T(S_{\text{list}}(\omega_2))$.
In such a case, multiple order fibers are effectively folded onto the same output fiber, reducing the set of fibers
hit by $S_{\text{list}}$ and potentially leaving other fibers uncovered.
When all the optimal solutions lie in missed fibers, the best achievable objective value within $\mathcal X(\mathcal P;S_{\text{list}})$
is necessarily larger than the global optimum, resulting in a positive optimality gap (cf.~Fig.~\ref{fig:illu-space}(c)).
\vspace{-4pt}
\begin{figure}
  \centering  \includegraphics[width=1.0\linewidth]{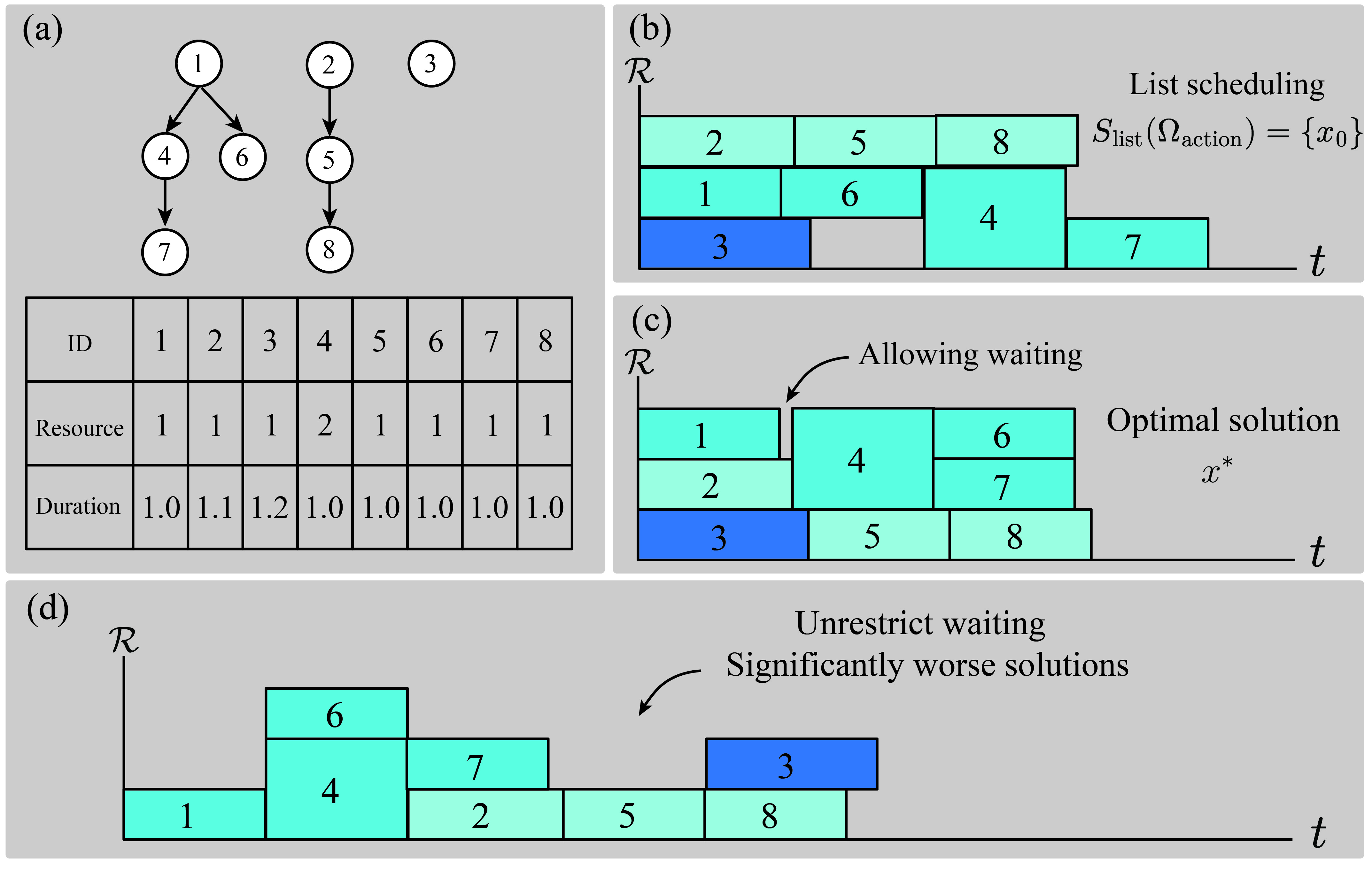}
  \caption{A counterexample showing list scheduling excludes the optimal solution for the DAG scheduling problem with heavy tasks. (a) Problem instance $\mathcal P_0$ with only one pool offering one type of resource with capacity 3. The resource demand and processing time is shown and the compatibility coefficients $K\equiv 1$. Each arrow $v\to w$ corresponds to an edge $(v,w)\in E$. (b) The only solution $x_0$ in the image of list scheduling $S_{\text{list}}(\Omega_{\text{action}}) = \{x_0\}$. The horizontal axis represents the time while the vertical axis represents the resource usage. (c) Optimal solution $x^*$ of the problem. (d) Worse solution for unrestricted waiting, leading to a significantly large objective value (makespan).}
  \label{fig:counterex}
\end{figure}

\subsection{Order-Covering Optimal Generation Maps}

A generation map may incur a positive optimality gap by missing the fiber that contains an optimal schedule.
We now introduce a criterion that rules out this issue by requiring that the map both (i) covers all feasible fibers and
(ii) selects a fiber-wise optimal representative.

\begin{definition}[Order-covering optimal (OCO) generation map]
\label{def:oco}
A generation map $S:\mathcal{W}\to\mathcal X$ is said to satisfy \emph{order-covering optimal} (OCO) condition if the following hold:
\begin{enumerate}
    \item[(O1)] (\emph{Order covering}) The composition $T\circ S:\mathcal W\to\mathcal B_f$ is surjective.
    \item[(O2)] (\emph{Fiber optimality}) For every $w\in\mathcal W$, the schedule $S(w)$ attains the minimal objective value
    among all schedules in the fiber of its induced order:
    \[
        f(x)\ \ge\ f(S(w)),\,\forall x\in\mathcal X \ \text{such that}\ T(x)=T(S(w)).
    \]
\end{enumerate}
\end{definition}

Under the fiber decomposition \eqref{eq:decomposition},
condition \textup{(O1)} ensures that for every feasible order $\hat w\in\mathcal B_f$ there exists some input $w$ such that
$S(w)\in T^{-1}(\hat w)$, i.e., the image $S(\mathcal W)$ intersects every feasible fiber.
Condition \textup{(O2)} further requires that, within each intersected fiber, the chosen point is a minimizer of $f$ over that fiber.
Hence, an OCO map cannot miss the global optimum due to incomplete fiber coverage.

\begin{proposition}[OCO maps eliminate optimality gaps]
\label{thm:oco-gap-zero}
If $S:\mathcal W\to\mathcal X$ satisfies the OCO condition, then
\[
\delta(\mathcal P,S)=0,\quad\text{for any instance $\mathcal{P}$}.
\]
Moreover, for any optimal solution $x^\star\in\mathcal X(\mathcal P)$, there exists an optimal solution
$\hat x\in S(\mathcal W(\mathcal P))$ such that $T(\hat x)=T(x^\star)$.
\end{proposition}

\begin{proof}
Since $TS$ is a surjection to $\mathcal B_f$ and $T(x)\in \mathcal B_f$, there exists $w\in \mathcal W$ such that $TS(w) = T(x)$. By the condition O2 we have $f(x)\ge f(S(w))$. Consequently, $S(w)\in \mathcal X(\mathcal P;S)$ is also an optimal solution. This finishes the proof.
\end{proof}

As shown in  Proposition \ref{pro:SGS-well-defined}, the serial schedule generation scheme $S_{\text{SGS}}$ enjoys a stronger property than \textup{(O1)}, namely $T\circ S_{\text{SGS}}=\mathrm{Id}_{\mathcal B_f}
$. This means that $S_{\text{SGS}}$ is an order-preserving section of $T$ on $\mathcal B_f$.
Consequently, $S_{\text{SGS}}$ establishes a one-to-one correspondence between the feasible order space $\mathcal B_f$
and the subset $S_{\text{SGS}}(\mathcal B_f)\subseteq \mathcal X$.
Moreover, each fiber intersects this subset at exactly one point:
\[
T^{-1}(w)\cap S_{\text{SGS}}(\mathcal B_f)=\{S_{\text{SGS}}(w)\},\quad \forall w\in\mathcal B_f,
\]
as illustrated in Fig.~\ref{fig:illu-space}(b). The next proposition further shows that $S_{\text{SGS}}$ satisfies \textup{(O2)}. Therefore, $S_{\text{SGS}}$ selects a fiber-wise optimal schedule over $T^{-1}(w)$ for each feasible schedule $w$.

\begin{proposition}
\label{pro:SGS-oco}
$S_{\text{SGS}}$ satisfies the OCO condition. Moreover, $T\circ S_{\text{SGS}}=\mathrm{Id}_{\mathcal B_f}$.
\end{proposition}

\begin{proof}
Given that the identity $T\circ S_{\text{SGS}}=\mathrm{Id}_{\mathcal B_f}$ follows from Proposition~\ref{pro:SGS-well-defined}, \textup{(O1)} holds. It remains to prove \textup{(O2)}. Fix $w\in\mathcal B_f$ and let $x=S_{\text{SGS}}(w)$.
Take any feasible schedule $\tilde x\in\mathcal X$ such that $T(\tilde x)=T(x)$.
We show that $f(x)\le f(\tilde x)$. Write $x=( s(\cdot), c(\cdot))$ and $\tilde x=(\tilde {{s}}(\cdot), c(\cdot))$ (the pool assignment coincides because $T(x)=T(\tilde x)$).
We claim that for every task $v\in V$,
\begin{equation}
\label{eq:sgs-dominate}
{s}(v)\le \tilde {{s}}(v),
\end{equation}
which implies $t_{\mathrm{end}}(v)\le \tilde t_{\mathrm{end}}(v)$ and then $f(x)\le f(\tilde x)$. To prove \eqref{eq:sgs-dominate}, consider the iteration of Algorithm~\ref{alg:optimal-section} when $v$ is scheduled on pool $c=c(v)$.
Let $\tilde\tau:=\tilde {{s}}(v)$. Since $T(\tilde x)=T(x)$, every task $u$ on pool $c$ with smaller within-pool rank than $v$
must satisfy $\tilde {{s}}(u)\le \tilde\tau$. As such tasks are scheduled before $v$ by $S_{\text{SGS}}$, and (by induction over the algorithm order)
their start times satisfy $ s(u)\le \tilde {{s}}(u)$, their completion times satisfy
$t_{\mathrm{end}}(u)\le \tilde t_{\mathrm{end}}(u)$.

Now compare the resource usage on pool $c$ at time $\tilde\tau$.
If a previously scheduled task $u$ is running at time $\tilde\tau$ in $x$ (i.e., $ s(u)\le \tilde\tau<t_{\mathrm{end}}(u)$),
it holds that $\tilde{ s}(u)\le \tilde\tau$ and $\tilde t_{\mathrm{end}}(u)\ge t_{\mathrm{end}}(u)>\tilde\tau$. $u$ is then also running at time $\tilde\tau$ in $\tilde x$.
Therefore, the set of tasks running on pool $c$ at time $\tilde\tau$ under $x$ is a subset of that under $\tilde x$,
and thus the available resource under $x$ at $\tilde\tau$ is component-wise no smaller than under $\tilde x$.
Since $\tilde x$ is feasible and starts $v$ at $\tilde\tau$, we have $\rho(v)\le \lambda_{\mathrm{cur}}(c)$ at time $\tilde\tau$ in the algorithmic state as well. Moreover, feasibility of $\tilde x$ implies $\tilde\tau\ge t_{\mathrm{dep}}(v)$ (all $E$-predecessors of $v$ complete before $\tilde\tau$),
and by construction $\tilde\tau$ is no earlier than the current pool time maintained by the algorithm.
Hence, $\tilde\tau$ is a feasible start time for $v$ when the algorithm schedules it.
Since Algorithm~\ref{alg:optimal-section} assigns $v$ the \emph{earliest} feasible start time on pool $c$, it follows that
$ s(v)\le \tilde\tau=\tilde{ s}(v)$, proving \eqref{eq:sgs-dominate}. Thus $f(S_{\text{SGS}}(w))\le f(\tilde x)$ for all $\tilde x$ with $T(\tilde x)=T(S_{\text{SGS}}(w))$, which is exactly \textup{(O2)}.
\end{proof}

By combining Proposition \ref{thm:oco-gap-zero} with Proposition \ref{pro:SGS-oco}, it follows that $S_{\text{SGS}}$ consistently removes the optimality gap for any instance. This provides a concrete reference principle for gap elimination. In particular, 
any generation map whose reachable schedules include the image of $S_{\text{SGS}}$ will inherently retain the global optimum. This principle directly informs our subsequent algorithm development. 

\section{Compatibility and Gap Aware Neural Scheduling}
\label{sec:generating-feasible-schedule}
\subsection{Overview: Two-Stage Single-Pass Scheduling}
\label{subsec:overview}
Guided by the order-space view, we design an optimality gap-free scheduling algorithm suitable for learning. 
The analysis therein reveals that the inherent gap of standard list scheduling $S_{\mathrm{list}}$ stems from an order-covering failure where the realized schedules miss feasible fibers containing optimal solutions. The serial SGS $S_{\text{SGS}}$ eliminates this gap by allowing waiting to ensure full fiber coverage. However, its strict order preservation often introduces excessive delays that slow down the convergence of learning. Our scheduling algorithm introduces waiting  in list scheduling by augmenting the action space with an active skip action. The resulting generation map is essentially an extension of $S_{\mathrm{list}}$ to the extended action sequence space. It effectively enlarges the image of $S_{\mathrm{list}}$ to include the image of $S_{\mathrm{SGS}}$, thus eliminating the optimality gap. Moreover, it allows waiting in a controlled way so that learning with this generation map will not be hampered.

In our constructive MDP view, the action selection is governed by a parameterized policy. This policy is formulated through task-pool scores and the skip score. Crucially, these quantities are generated in a single-pass regime, where the  neural network is executed once per instance and is not re-invoked during rollout. The single-pass design significantly boosts the efficiency and scalability of our framework.

To capture the environmental heterogeneity, we propose a weighted cross-attention network (WeCAN) to generate the quantities required by the scheduling algorithm. This network is  compatibility aware through  a weighted cross-attention (WeCA) mechanism to explicitly model task-pool interactions. Notably, this mechanism is size-agnostic, enabling a single trained network to effectively tackle fluctuating environments with a varying number of task types and resource pools.

Putting the above together yields a two-stage process. Stage~I performs a single-pass compatibility-aware network processing for scoring, and Stage~II applies a score-driven skip-extended scheduling algorithm to obtain feasible schedules without generation-induced optimality gaps. The overall framework of our two-stage scheduling process is illustrated in Fig.~\ref{fig:framework}.

\subsection{Score-Driven Skip-Extended Scheduling Algorithm}

Stage~II turns the scores and skip parameters into a concrete feasible schedule via Algorithm~\ref{alg:learning},
which extends the classic list-scheduling  paradigm by allowing active skips.
Algorithm~\ref{alg:learning} takes as input an instance $\mathcal P=(V,E,\mathcal C,\rho,t,\lambda,K_{\mathrm{acc}})$,
task-pool action scores $\{u_{(v,c)}\}_{(v,c)\in\mathcal A}$, and a skip-parameter vector $\psi$
from which a step-dependent skip score $u_{\text{skip}}(k;\psi)$ is computed on the fly,
and outputs a feasible schedule $x=( s(\cdot), c(\cdot))$.
In \textsc{Greedy} mode, the output is deterministic for fixed $(u,\psi)$; in \textsc{Sampling} mode, the same inputs induce
a distribution over feasible schedules through randomized action selection. Under the hood, Algorithm~\ref{alg:learning} implicitly induces a generation map $S_d$ defined on the extended action sequence space $\overline{\Omega}_{\text{action}}$, whose formal definition is available in Appendix \ref{apx:theoritical results}.

\smallskip
\noindent\textbf{Active skip and decreasing skip score.}
The key extension beyond vanilla list scheduling is an active skip action $a_{\text{skip}}$ that advances the current time
to the next completion event even when ready tasks exist.
Accordingly, we assign $a_{\text{skip}}$ a step-dependent score so that it can be selected together with task-pool actions under the same score-to-policy rule:
\begin{equation}
\label{eq:skip-score}
    \begin{aligned}
    q_{\text{skip}}(k)\;&=\;\alpha\exp\Bigl(-\gamma k/(2n)\Bigr)+\beta,\\
u_{\text{skip}}(k) \;&=\; \log(q_{\text{skip}}(k)),
\end{aligned}
\end{equation}
where $n=|V|$, $k$ is the number of decisions taken so far and $\alpha,\beta,\gamma>0$ are positive skip parameters. For conciseness, 
we write $\psi:=(\alpha,\beta,\gamma)$. This decreasing functional form serves to mitigate over-prioritization of the skip action and thus eliminates the structural optimality gap without hampering the learning process. A constant skip score is insufficient to achieve this since the  resulting reachable set may fail to contain the image of $S_{\text{SGS}}$ for gap elimination. This issue is analyzed in Section~\ref{subsec:ablation-skip}. Re-evaluating the skip action via the neural network at every decision step to obtain dynamic scores would necessitate multiple forward passes. To reconcile with the single-pass regime, we instead  predict $\psi$ once and compute $u_{\text{skip}}(k;\psi)$ analytically during rollout. 

\smallskip
\noindent\textbf{Feasibility masking and action selection.}
At each decision step, Algorithm~\ref{alg:learning} masks task-pool actions that are infeasible under the current partial schedule,
including violations of (i) precedence constraints, (ii) pool capacity constraints, and (iii) compatibility constraints. Incompatibility corresponds to $K_{\mathrm{acc}}(v,c)=0$.
The skip action is masked only when no task is currently running (so that no completion event exists).
Given scores and masks $M_a$, we form a distribution over the extended action set
$\bar{\mathcal A}=\mathcal A\cup\{a_{\text{skip}}\}$ via
\begin{equation}
\label{eq:prob}
p_\theta(a_t=a)=\frac{\exp(u_a+M_a)}{\sum_{\tilde a\in\bar{\mathcal A}}\exp(u_{\tilde a}+M_{\tilde a})}, \quad\text{for all $a\in\bar{\mathcal{A}}$}.
\end{equation}
In \textsc{Greedy} mode we select $a_t\in\arg\max_{a\in\bar{\mathcal A}}(u_a+M_a)$ to obtain a deterministic schedule,
whereas in \textsc{Sampling} mode we sample actions for exploration during training.

\begin{algorithm}[tb]
   \caption{Score-driven skip-extended scheduling algorithm}
   \label{alg:learning}
\begin{algorithmic}[1]
   \STATE {\bfseries Input:} Instance $\mathcal P=(V,E,\mathcal C,\rho,t,\lambda,K_{\mathrm{acc}})$; action scores $\{u_{(v,c)}\}_{(v,c)\in\mathcal A}$; skip-parameter set $\psi$; mode \textsc{Greedy} / \textsc{Sampling}.
   \STATE {\bfseries Output:} A feasible schedule $x=( s(\cdot), c(\cdot))$.
   \STATE Initialize current time $t_{\mathrm{cur}}\gets 0$, unscheduled tasks $U\gets V$, and pool states (available resources and running tasks).
   \WHILE{$U\neq \emptyset$}
      \STATE Compute the ready set $R$ at time $t_{\mathrm{cur}}$.
      \STATE Apply feasibility masks to scores $\{u_{(v,c)}\}$ according to dependency/resource/compatibility constraints at $t_{\mathrm{cur}}$.
      \STATE Set the skip score $u_{\mathrm{skip}}$ by the rule \eqref{eq:skip-score}; mask $a_{\mathrm{skip}}$ iff no task is currently running.
      \STATE Form the policy over $\bar{\mathcal A}=\mathcal A\cup\{a_{\mathrm{skip}}\}$ by \eqref{eq:prob}, and select an action $a_t$ by the specified mode.
      \IF{$a_t=(v,c)\in\mathcal A$}
          \STATE Start task $v$ on pool $c$ at time $t_{\mathrm{cur}}$: set $ s(v)\gets t_{\mathrm{cur}}$, $ c(v)\gets c$.
          \STATE Update pool state and remove $v$ from $U$.
      \ELSE
          \STATE Advance time to the next completion event and release all tasks completing at the new $t_{\text{cur}}$.
      \ENDIF
      \STATE Update dependency status and pool states.
   \ENDWHILE
\end{algorithmic}
\end{algorithm}

\begin{theorem}
\label{thm1}
Let $n=|V|$. Algorithm~\ref{alg:learning} satisfies:
\begin{enumerate}
\item The rollout terminates in at most $2n$ decisions and outputs a feasible schedule.
\item In \textsc{Sampling} mode, for any real-valued scores $\{u_{(v,c)}\}$ and parameters $\psi$,
the induced distribution assigns strictly positive probability to every schedule in the image
$\mathcal X(\mathcal P; S_{\text{SGS}})$ of the serial schedule generation scheme. In particular, it assigns positive probability to
(i) at least one schedule realizing any given feasible order in $\mathcal B_f(\mathcal P)$,
and (ii) at least one optimal schedule of $\mathcal P$.

\item In \textsc{Greedy} mode, for each instance $\mathcal P$ there exist scores
$\{u_{(v,c)}\}$ and parameters $\psi$ such that Algorithm~\ref{alg:learning} outputs an optimal schedule.
\item For the counterexample instance $\mathcal P_0$, statements (2)--(3) fail if $a_{\text{skip}}$ is removed.
\end{enumerate}
\end{theorem}

\begin{proof}
We prove the four statements in order.

(1) Termination and feasibility.
Feasibility follows from the masking rule: at any time $t_{\mathrm{cur}}$,
only actions that respect precedence constraints in $E$, pool-capacity constraints, and compatibility constraints
are unmasked. Hence, any dispatched action produces a feasible partial schedule. For termination, since each dispatch action schedules exactly one previously unscheduled task, at most $n=|V|$ dispatches can occur.
Given that each skip action advances $t_{\mathrm{cur}}$ to the next completion event and at least one running task completes at that event, the total number of skips is at most the number of task completions, which is at most $n$.
Therefore, the rollout terminates in at most $2n$ decisions and outputs a feasible schedule.

(2) Strictly positive probability for every schedule in $\mathcal X(\mathcal P;S_{\text{SGS}})$.
Under the sampling policy in~\eqref{eq:prob}, for any step $t$, every unmasked action has strictly
positive probability since $\exp(\cdot)>0$ and the denominator is finite and positive.
Therefore, it suffices to show that for any schedule $x\in \mathcal X(\mathcal P;S_{\text{SGS}})$, there exists
a step-by-step feasible action sequence that realizes $x$ under Algorithm~\ref{alg:learning}.

Fix any $x\in \mathcal X(\mathcal P;S_{\text{SGS}})$. By definition of the image, there exists a feasible order
$\hat w\in \mathcal B_f(\mathcal P)$ such that $x = S_{\text{SGS}}(\hat w)$.
Consider the SGS procedure (Algorithm~\ref{alg:optimal-section}) on $\hat w$ that produces $x$.
This run induces an \emph{SGS trace}:
an ordered list of dispatched tasks $(v_1,\dots,v_n)$ with pool assignments $c(v_i)$ and start times
$\tau_i := s_x(v_i)$ (the start time of $v_i$ in $x$).
A key property of SGS is that each $\tau_i$ is an \emph{event time} of the partial schedule built so far:
$\tau_i=0$ for the first dispatch(es), and whenever $\tau_i>0$ it equals the completion time of at least one
previously started task (since SGS sets the start time as $t_{\text{dep}}$ or $t_{\text{cur}}(c)$, while both being task completion time).

We now construct an action sequence for Algorithm~\ref{alg:learning} that realizes $x$.
Maintain an index $i$ pointing to the next unscheduled task in the trace.
At any decision time $t_{\mathrm{cur}}$:
\begin{itemize}
\item If $t_{\mathrm{cur}} = \tau_i$, dispatch $(v_i,c(v_i))$ and increment $i$.
(If multiple consecutive tasks in the trace share the same start time $\tau_i=t_{\mathrm{cur}}$,
dispatch them in that trace order before taking any skip.)
\item If $t_{\mathrm{cur}} < \tau_i$, take $a_{\mathrm{skip}}$.
\end{itemize}
We claim every chosen action above is unmasked, and the resulting rollout reproduces $x$.

\emph{Dispatch steps are unmasked.}
When $t_{\mathrm{cur}}=\tau_i$ and we dispatch $(v_i,c(v_i))$, all precedence constraints of $v_i$
are satisfied because $x$ is feasible and starts $v_i$ at $\tau_i$.
Moreover, by induction the partial schedule constructed by Algorithm~\ref{alg:learning} matches $x$ on already dispatched tasks,
then the set of running tasks on each pool at time $\tau_i$ (also the remaining pool capacity and compatibility status)
coincides with that in $x$. Since $x$ starts $v_i$ at $\tau_i$ on pool $c(v_i)$, the resource/compatibility masks also allow $(v_i,c(v_i))$.
Thus, $(v_i,c(v_i))$ is unmasked.

\emph{Skip steps are unmasked and do not overshoot $\tau_i$.}
When $t_{\mathrm{cur}}<\tau_i$, by the event-time property of SGS, the next start time $\tau_i$ cannot be strictly
after $t_{\mathrm{cur}}$ while having no running task on every pool; otherwise the system would be idle and SGS would be able to
start the next trace task at $t_{\mathrm{cur}}$, contradicting $\tau_i>t_{\mathrm{cur}}$.
Hence, at least one task is running and $a_{\mathrm{skip}}$ is unmasked. Furthermore, $a_{\mathrm{skip}}$ advances $t_{\mathrm{cur}}$ to the next completion event among currently running tasks,
which is the next event time of the (matching) partial schedule and cannot exceed the next SGS dispatch time $\tau_i$.
Repeating skip finitely many times reaches $t_{\mathrm{cur}}=\tau_i$.

Combining the two cases, the above construction yields a well-defined rollout that realizes $x$, and every chosen action is unmasked.
Since unmasked actions have strictly positive sampling probability under \eqref{eq:prob},
the probability of this rollout (also for generating $x$) is strictly positive.
This proves that, in \textsc{Sampling} mode, the induced distribution assigns strictly positive probability to every
$x\in \mathcal X(\mathcal P;S_{\text{SGS}})$.

The two stated consequences follow immediately:
(i) for any feasible order $\hat w\in\mathcal B_f(\mathcal P)$, the schedule $S_{\text{SGS}}(\hat w)$ lies in the image,
then has positive probability; and
(ii) by the OCO property of $S_{\text{SGS}}$ there exists at least one optimal schedule in $\mathcal X(\mathcal P;S_{\text{SGS}})$,
which also has positive probability.

(3) Existence of scores yielding an optimal schedule in \textsc{Greedy} mode.
Fix the optimal schedule $\hat x\in \mathrm{Im}(S_{\text{SGS}})$ constructed above, and fix a feasible targeted rollout
$\omega^\star=(a_1,\dots,a_N)$ that realizes $\hat x$ under Algorithm~\ref{alg:learning}, where each $a_t$ is unmasked at its step.
We construct static task--pool scores $\{u_{(v,c)}\}$ and skip parameters $\psi$ such that \textsc{Greedy} mode reproduces $\omega^\star$
and therefore outputs $\hat x$.

Choose $\psi=(\alpha,\beta,\gamma)$ with $\alpha>0$ and $\gamma>0$ so that the skip score $u_{\mathrm{skip}}(t)$ defined in~\eqref{eq:skip-score}
is strictly decreasing for $t=1,\dots,N$.
Set
\[
\Delta:=\min_{1\le t\le N-1}\bigl(u_{\mathrm{skip}}(t)-u_{\mathrm{skip}}(t+1)\bigr)\;>\;0,
\qquad
\varepsilon:=\Delta/4.
\]
Define static task--pool scores as follows. For every dispatch step $t$ with $a_t=(v_t,c_t)\in\mathcal A$, set $
u_{(v_t,c_t)}:=u_{\mathrm{skip}}(t)+2\varepsilon$. For every other task--pool pair $(v,c)\in\mathcal A$ that does not appear as a dispatch action in $\omega^\star$, set 
$u_{(v,c)}:=u_{\mathrm{skip}}(N)-2\varepsilon$. We claim that under these scores, \textsc{Greedy} mode selects $a_t$ at each step $t$.
Proceed by induction on $t$. Assume the first $t-1$ decisions coincide with $\omega^\star$; then the algorithm is in the same state
as the targeted rollout at step $t$. Hence, $a_t$ is unmasked and the set of unmasked actions is identical.
Consider two cases.

If $a_t=a_{\mathrm{skip}}$, then for any unmasked dispatch action $(v,c)$, either $(v,c)$ appears as a dispatch action
in $\omega^\star$ at some later index $j>t$, in which case
\[
u_{(v,c)}=u_{\mathrm{skip}}(j)+2\varepsilon \le u_{\mathrm{skip}}(t+1)+2\varepsilon
< u_{\mathrm{skip}}(t)=u_{\mathrm{skip}}(t),
\]
where the strict inequality follows from $u_{\mathrm{skip}}(t)-u_{\mathrm{skip}}(t+1)\ge \Delta>2\varepsilon$,
or $(v,c)$ never appears in $\omega^\star$, in which case $u_{(v,c)}=u_{\mathrm{skip}}(N)-2\varepsilon<u_{\mathrm{skip}}(t)$.
Therefore, $a_{\mathrm{skip}}$ has strictly larger score than any other unmasked action and is selected by \textsc{Greedy} mode.

If $a_t=(v_t,c_t)\in\mathcal A$, then $u_{(v_t,c_t)}=u_{\mathrm{skip}}(t)+2\varepsilon>u_{\mathrm{skip}}(t)$, so it beats $a_{\mathrm{skip}}$.
For any other unmasked dispatch action $(v,c)$, either it appears later in $\omega^\star$ at some $j>t$, in which case
$u_{(v,c)}=u_{\mathrm{skip}}(j)+2\varepsilon<u_{\mathrm{skip}}(t)+2\varepsilon=u_{(v_t,c_t)}$,
or it never appears in $\omega^\star$, in which case $u_{(v,c)}=u_{\mathrm{skip}}(N)-2\varepsilon<u_{(v_t,c_t)}$.
Thus, $(v_t,c_t)$ has strictly largest score among unmasked actions and is selected by \textsc{Greedy} mode.

In both cases, \textsc{Greedy} selects $a_t$, completing the induction. Hence, \textsc{Greedy} reproduces the entire rollout $\omega^\star$
and outputs $\hat x$, which is optimal.

(4) Failure without the skip action on $\mathcal P_0$.
Removing $a_{\mathrm{skip}}$ reduces Algorithm~\ref{alg:learning} to the standard list-scheduling realization rule.
For the counterexample instance $\mathcal P_0$,
list scheduling cannot reach any optimal schedule. Consequently, statements (2)--(3) fail on $\mathcal P_0$ without skip.

This completes the proof.
\end{proof}

\subsection{Compatibility-Aware WeCAN Architecture}\label{sec:network}

\begin{figure*}
  \centering
  \includegraphics[width=0.95\linewidth]{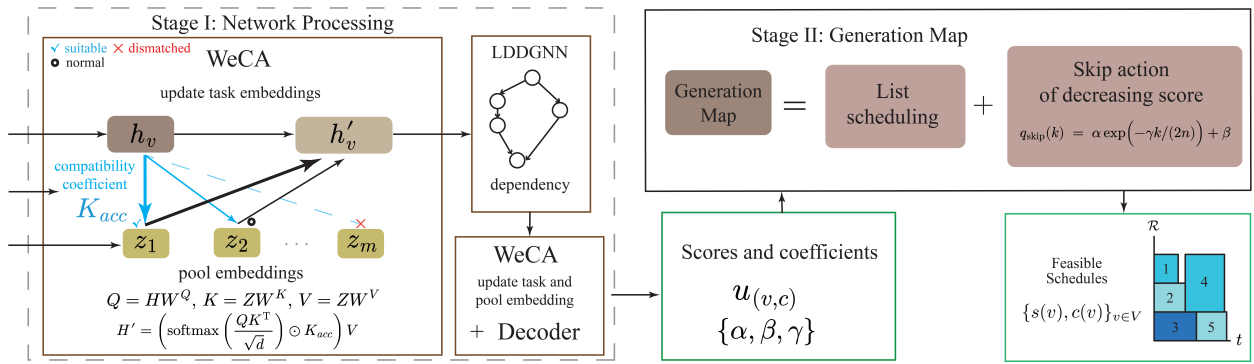}
  \caption{The two-stage framework of our architecture. The first stage involves once network processing by our efficient and adaptive weighted cross-attention architecture. It embeds the problem attribution and provides scores and coefficients for the second stage. The second stage employs a generation map extended from list scheduling. This extension restores surjectivity onto the schedule order space and ensures the inclusion of an optimal solution in the image of the generation map, without significantly slowing down training convergence.}
  \label{fig:framework}
\end{figure*}

In Stage~I, WeCAN generates the scores and skip parameters required by Algorithm~\ref{alg:learning}. This network adopts a standard encoder-decoder architecture. The encoder encodes the scheduling problem into embeddings, and the decoder converts these embeddings into the scores and skip parameters that drive the scheduling procedure. We now elaborate on the architecture of the encoder and decoder, respectively. 

The encoder is designed to extract expressive embeddings from the heterogeneous scheduling problem, by processing task-resource interactions and the DAG structure. Specifically, each task node $v$ is characterized by a feature vector $(t(v), \rho(v))$ representing its processing time and resource demands, while each resource pool $c$ is defined by its capacity vector $\lambda(c)$. These raw features are first turned into two sets of initial embeddings, i.e., task embeddings $\{\bm{h}_v^{(0)}\}_{v\in V}$ and pool embeddings $\{\bm{z}_c^{(0)}\}_{c\in\mathcal{C}}$, through two separate multi-layer perceptrons (MLPs). The initial embeddings capture only isolated characteristics of tasks and pools, failing to capture complicated environmental heterogeneity and intricate task dependencies inherent in the problem. To resolve the former, we propose the WeCA mechanism to synthesize heterogeneous environmental information, especially the varying task-resource compatibilities. Notably, this mechanism remains adaptable regarding the number of task types and resource pools, without changing the network architecture. To address the latter, a longest directed distance-based GNN is applied to encode the intricate precedence dependencies encompassed within the DAG structure. Below we present more details about these two main components in our encoder. 

\textbf{WeCA.} The WeCA layer executes weighted message passing through a gating  mechanism, enabling compatibility-aware interactions between tasks and resource pools. We focus here on the directional flow from resource pools to tasks. The reverse direction is defined symmetrically using the transposed compatibility matrix $K_{acc}^{\rm T}$. Let $\bm{H}\in\mathbb{R}^{|V|\times d}$  and $\bm{Z}\in\mathbb{R}^{|\mathcal C|\times d}$ denote the stacked embedding matrices for tasks and resource pools, respectively. The WeCA operation is formalized as:
    \begin{align*}
    &\bm{Q} = \bm{H}\bm{W}^Q,\, \bm{K}=\bm{Z}\bm{W}^K,\,\bm{V}=\bm{Z}\bm{W}^V,\\
    &\mathrm{WeCA}(\bm{H},\bm{Z},K_{acc}) := \left(\mathrm{softmax}\left(\frac{\bm{Q}\bm{K}^{\rm T}}{\sqrt{d}}\right)\odot K_{acc}\right)\bm{V}.
\end{align*}
In practice, the multi-head and skip-connection techniques are also incorporated, as in the vanilla Transformer architecture ~\cite{attention}, to enhance the expressive power and training stability. By integrating the compatibility coefficients directly into the attention weights, the WeCA mechanism allows each task to selectively attend to pools that are both feasible and advantageous under heterogeneous execution conditions. Specifically, this mechanism evaluates potential pools for each task and aggregates their features, weighted by compatibility. Recall that task $v$ cannot run on pool $c$ if and only if $K_{acc}(v,c) = 0$. Thus, incompatible pool assignments are inherently masked in the above attention calculation. Moreover, a larger $K_{acc}(v,c)$ indicates that task $v$ runs faster on pool $c$, indicating better compatibility. Therefore, for each task, the WeCA layer effectively gathers and integrates information primarily from its most compatible pools. Crucially, unlike approaches relying on fixed-size embeddings, this mechanism maintains adaptability without imposing fixed dimensionality constraints on the network architecture. This enables the framework to evaluate context-dependent resource capacity for each task more realistically within the heterogeneous environment.

A critical design choice in WeCA is applying the compatibility coefficient matrix as a multiplicative gating factor outside the softmax normalization, rather than as an additive logarithmic attention bias within it. For comparison, an inside variant would be formulated as:
\begin{equation*}
\text{WeCA}_{\text{inside}}(\bm{H},\bm{Z},K_{acc})
=\mathrm{softmax}\left(\frac{\bm{Q}\bm{K}^{\rm T}}{\sqrt{d}} + \log K_{acc}\right)\bm{V}, 
\end{equation*}
where $\log K_{acc}$ serves as the bias term. We opt for the former specifically to preserve information regarding the overall compatibility degree of tasks in the heterogeneous scheduling scenario. Consider a scenario of two pools with identical capacity and two tasks $v_1,v_2$ with the same attribute. Suppose $v_1$ is compatible with only one pool while $v_2$ is compatible with both pools. Despite the identical attribute, these tasks differ significantly in their environmental compatibility. If the inside placement were taken, the normalization effect could lead to the same embeddings for both tasks, failing to distinguish their different compatibility profiles. Conversely, the outside placement better reflects a task's overall compatibility within the environment, resulting in a more accurate and distinguishable embedding.

\textbf{LDDGNN.}
To encode precedence information in DAGs, we use an attention-based GNN, as standard message-passing GNNs often struggle to capture long-range directed dependencies.
Our key design choice is to parameterize attention by the longest directed distance (LDD) $d_e(v,w)$. This metric is positive when $v$ reaches $w$ and negative when $w$ reaches $v$, with the magnitude being the length of the longest directed path. For the case without a directed path, the LDD is classified as incomparable ($+\infty$) or disconnected ($-\infty$) according to connectivity in the undirected graph.
The motivation is that $|d_e(v,w)|$ provides a simple proxy for dependency depth in the DAG. Among multiple directed connections, the longest path best distinguishes the levels of precedence separation, which is crucial for modeling delay propagation along critical paths.
In the spirit of Graphormer~\cite{Ying22nips} and Topoformer~\cite{Gag22nips}, we inject LDD into attention through an LDD-dependent bias and mask.
Concretely, in head $j$ we modify the attention logit as
\begin{equation*}
\begin{aligned}
    \displaystyle \alpha_{v,w}^j = \frac{\langle \displaystyle \bm{q}_v^j, \displaystyle \bm{k}_w^j\rangle}{\sqrt{d}} + b_{d_e(v,w)}+ M_{v,w}^j,
\end{aligned}
\end{equation*}
where $b_{d_e(v,w)}$ is a learnable scalar indexed by $d_e(v,w)$ and $M_{v,w}^j$ is an attention mask. Precise definitions of the LDD, the bias term, and the attention mask are provided in Appendix \ref{apx:implementation_details}.1. The remaining computations follow a standard multi-head attention block. This formulation enables each head to selectively attend to structurally relevant relations, such as ancestors or descendants within a desired depth range.

In the full encoder, task embeddings are first updated by one WeCA layer and a LDDGNN. The encoder then alternately updates the task embeddings and the pool embeddings using WeCA layers, so that information propagates back and forth between tasks and pools.

Given the final task embeddings $\{\bm{h}^{(L)}_v\}_{v\in V}$ and pool embeddings $\{\bm{z}^{(L)}_c\}_{c\in\mathcal{C}}$ from the encoder, the decoder then generates task-pool action scores and skip parameters as follows. The score of each task-pool pair $(v,c)$ is calculated through a weighted inner product:
\begin{align*}
             u_{(v,c)} &=  (\bm{W}^Q_{s} \bm{h}_v^{(L)})^{\rm T}(\bm{W}^K_{s} \bm{z}_c^{(L)}) + \log(K_{acc}(v,c)).
\end{align*}
In addition, skip parameters $\psi=\mathrm{MLP}\left(\frac{1}{|V|}\sum_{v\in V}\bm{h}_v\right)$ are derived from the average task embedding with an MLP. 

\subsection{Learning with a parameterized probabilistic model}

Our framework learns a parameterized probabilistic model to minimize the expected makespan of realized schedules over the distribution of problem instances.
Specifically, the probabilistic model $p_{\theta}(\omega|\mathcal{P})$ is defined in \eqref{eq:prob-mdp} with the action policy $p_{\theta}(a_t| s_t,\mathcal{P})$ specified in \eqref{eq:prob}. Crucially, this action policy is formulated through scores and skip parameters that are computed once by the network from  the instance $\mathcal{P}$ without re-encoding the evolving state $s_t$. This single-pass parameterization substantially improves efficiency and scalability, without sacrificing the solution quality, which is empirically demonstrated in Appendix \ref{apx:narvsar}. 

To learn the probabilistic model over action sequences, we use a reinforcement learning formulation in which each sampled sequence is evaluated by the makespan of its realized schedule. Let $\omega$ denote an action sequence sampled from $p_\theta(\cdot\mid \mathcal P)$ and let $S_d(\omega)$ be the schedule realized by our generation map. The sampling and realization processes are jointly accomplished by Algorithm~\ref{alg:learning}. 
By denoting $\hat f(\omega):= f\!\big(S_d(\omega)\big)$, the optimization problem for learning is formulated as:
\[
\min_\theta\; L(\theta)=\mathbb E_{\mathcal P}\Big[\mathbb E_{\omega\sim p_\theta(\cdot\mid \mathcal P)}\big[\hat  f(\omega)\big]\Big].
\]
The gradient of $L(\theta)$ is given by 
\begin{equation*}
\nabla L(\theta)
=\mathbb{E}_{\mathcal P}\!\left[\mathbb{E}_{\omega\sim p_\theta(\cdot\mid \mathcal P)}
\left[(\hat f(\omega)-b(\mathcal P))\,\nabla_\theta \log p_\theta(\omega \mid \mathcal P)\right]\right].
\end{equation*}
Here $b(\mathcal P)$ is a baseline, e.g., the instance-wise average makespan, to reduce the variance. We then solve this problem using the policy gradient algorithms in RL.

\begin{table*}[htbp]
  \caption{Experimental results on TPC-H datasets.}
  \label{table:tpch-main}
  \setlength{\tabcolsep}{3.4pt}
  \centering
  \begin{tabular}{lccccccccc}
    \toprule
    & \multicolumn{3}{c}{TPC-H-30, 3 pools} & \multicolumn{3}{c}{TPC-H-50, 3 pools} & \multicolumn{3}{c}{TPC-H-100, 3 pools} \\
    & Makespan & Improvement & Time & Makespan & Improvement & Time & Makespan & Improvement & Time \\
    \midrule
    SFT     & 27404 & -28.12\% & 0.23 & 49172 & -32.31\% & 0.78 & 84986 & -24.27\% & 3.08 \\
    MOPNR   & 25052 & -17.13\% & 0.30 & 43545 & -17.17\% & 0.99 & 77362 & -13.12\% & 3.34 \\
    CP      & 23869 & -11.59\% & 0.29 & 41597 & -11.93\% & 0.90 & 74364 & -8.74\%  & 3.35 \\
    HEFT    & 23177 & -8.36\%  & 0.18 & 39315 & -5.79\%  & 0.54 & 70137 & -2.56\%  & 1.86 \\
    Tetris  & 23170 & -8.33\%  & 0.21 & 38654 & -4.01\%  & 0.62 & 71296 & -4.25\%  & 2.13 \\
    PEFT    & 22930 & -7.20\%  & 0.18 & 39881 & -7.31\%  & 0.46 & 69831 & -2.11\%  & 1.43 \\
    IPPTS   & 21389 & 0.00\%   & 0.11 & 37164 & 0.00\%   & 0.34 & 68387 & 0.00\%   & 1.33 \\
    \midrule
    PPO-BiHyb              & 19788 & 7.49\%  & 33.07  & 35085 & 5.59\%  & 59.48  & 66436 & 2.85\%  & 179.14 \\
    One-Shot-S(256)        & 20399  & 4.63\%  & 2.26 & 35561  & 4.31\%  & 4.16 & 66173  & 3.24\%  & 9.85 \\
    ScheduleNet-S(256)     & 21474   & -0.40\% & 2.02 & 36607   & 1.50\%  & 4.04 & 69644  & -1.84\% & 9.43 \\
    HGN-Two-Stage-S(256)   & 21068  & 1.50\%  & 2.08 & 37068  & 0.26\%  & 4.11 & 69670  & -1.88\% & 9.05 \\
    \midrule
    WeCAN-Greedy           & 19578 & 8.47\%  & 0.15  & 33428 & 10.05\% & 0.50  & 62587 & 8.48\%  & 1.72 \\
    WeCAN-S(64)            & 19053  & 10.92\% & 1.54 & 32912  & 11.44\% & 2.86 & 61662  & 9.83\%  & 5.26 \\
    WeCAN-S(256)           & \textbf{18964 } & \textbf{11.34\%} & 2.43 & \textbf{32814 } & \textbf{11.70\%} & 4.39 & \textbf{61373 } & \textbf{10.26\%} & 10.43 \\
    \bottomrule
  \end{tabular}
\end{table*}
\begin{table*}[htbp]
\caption{Experimental results on computation graph datasets with 500 tasks.}
\label{table:ali-main}
\centering
\setlength{\tabcolsep}{3.6pt}
\begin{tabular}{lccccccccc}
\toprule
& \multicolumn{3}{c}{Erd\H{o}s-R\'{e}nyi} & \multicolumn{3}{c}{Layer Graphs} & \multicolumn{3}{c}{Stochastic Block} \\
& Makespan & Improvement & Time & Makespan & Improvement & Time & Makespan & Improvement & Time \\
\midrule
SFT     & 13317 & -19.99\% & 0.81 & 16158 & -30.04\% & 0.34 & 14408 & -27.96\% & 0.53 \\
MOPNR   & 12771 & -15.07\% & 1.07 & 14714 & -18.42\% & 0.38 & 13148 & -16.77\% & 0.68 \\
Tetris  & 13084 & -17.90\% & 0.52 & 14271 & -14.86\% & 0.44 & 13666 & -21.37\% & 0.64 \\
CP      & 12457 & -12.25\% & 1.08 & 14797 & -19.09\% & 0.40 & 13388 & -18.90\% & 0.74 \\
PEFT    & 11311 & -1.92\%  & 0.45 & 12425 & 0.00\%   & 0.71 & 11438 & -1.58\%  & 0.56 \\
IPPTS   & 11679 & -5.24\%  & 0.57 & 16710 & -34.49\% & 0.64 & 12490 & -10.92\% & 0.49 \\
HEFT    & 11098 & 0.00\%   & 0.55 & 12428 & -0.02\%  & 0.75 & 11260 & 0.00\%   & 0.57 \\
\midrule
PPO-BiHyb            & 10795 & 2.73\%  & 65.51 & 11883 & 4.36\%  & 73.7 & 10885 & 3.33\%  & 73.7 \\
One-Shot-S(256)      & 11071  & 0.24\%  & 4.45 & 12277  & 1.19\%  & 3.83 & 11377  & -1.04\% & 4.00 \\
ScheduleNet-S(256)   & 11355  & -2.32\% & 4.16 & 11923  & 4.04\%  & 4.04 & 11076  & 1.63\%  & 4.06 \\
HGN-Two-Stage-S(256) & 11434  & -3.03\% & 4.05 & 12316  & 0.88\%  & 3.91 & 11712  & -4.01\% & 3.95 \\
\midrule
WeCAN-Greedy         & 10270 & 7.46\%  & 0.57 & 11173 & 10.08\% & 0.26 & 10539 & 6.40\%  & 0.41 \\
WeCAN-S(64)          & 10115  & 8.86\%  & 3.21 & 10862  & 12.58\% & 3.07 & 10074  & 10.53\% & 3.06 \\
WeCAN-S(256)         & \textbf{10083 } & \textbf{9.15\%} & 4.94 & \textbf{10752 } & \textbf{13.46\%} & 4.30 & \textbf{10019 } & \textbf{11.02\%} & 4.58 \\
\bottomrule
\end{tabular}
\vspace{-5pt}
\end{table*}
\begin{table*}[!t] 
    \caption{Full results on RI-W datasets.}
    \label{table-heavy-full}
    \centering
    {
        \begin{tabular}{lcccccc}
        \toprule
         & \multicolumn{3}{c}{RI-W-30} & \multicolumn{3}{c}{RI-W-50} \\
         & Makespan & Improvement & Time & Makespan & Improvement & Time \\
        \midrule
        Tetris & 29324 & -14.11\% & 0.20 & 44296 & -16.46\% & 0.56\\
        CP & 28144 & -9.52\% & 0.26 & 41378 & -8.79\% & 0.75\\
        \midrule
        HEFT & 26856 & -4.51\% & 0.18 & 39825 & -4.71\% & 0.52 \\
        PEFT & 28260 & -9.97\% & 0.13 & 39928 & -4.98\% & 0.48 \\
        IPPTS & 25697 & 0.00\% & 0.21 & 38034 & 0.00\% & 0.51 \\
        \midrule
        PPO-BiHyb & 24383 & 5.11\% & 33.74 & 36958 & 2.83\% & 55.47 \\
        OneShot-S(256) & 27425  & -6.72\%  & 2.12 & 39826  & -4.71\%  & 4.02 \\
        \midrule
        WeCAN-S(256) & \textbf{24099 } & \textbf{6.22\% } & 2.50 & \textbf{35695 } & \textbf{6.15\% } & 4.72 \\
        \bottomrule
        \end{tabular}
    }
\end{table*}

\section{Numerical Experiments}\label{sec:experiments}

\subsection{Experimental Setup}
All experiments are conducted on an internal cluster with NVIDIA Tesla A100 (80 GB), 4$\times$Intel Xeon Silver 4310 \@ 2.10GHz (12 cores) and 192GB memory.

\textbf{Datasets.} We evaluate the proposed method on three datasets that span representative application settings in DAG scheduling, including query plans, computation graphs and resource-intensive workloads. 
i) \textit{TPC-H dataset:} a  dataset that comprises  real-world DAG tasks derived from industrial queries, reflecting DAG scheduling workloads in data-intensive query processing systems. We use the version curated by \cite{Wang21nips} and augment it with randomly generated memory demands and additional task types, where each type is associated with a set of compatibility coefficients. The problems in TPC-H-30, TPC-H-50, and TPC-H-100 contain 275, 459, and 918 tasks on average.
 ii) \textit{computation graph dataset:} a synthetic dataset generated using approaches from \cite{jeon2023iclr}, comprising neural-network computation graphs arising in ML compilers and runtime optimization, including layer graphs, Erdos-Rényi graphs, and stochastic block model graphs. Each problem contains 500 tasks. iii) \textit{resource-intensive workload (RI-W) dataset:} a dataset designed to reflect resource-intensive scheduling scenarios, where a small fraction of ``heavy tasks'' have high resource demand and long processing time. This scenario is prevalent in contemporary cloud computing platforms that provide heterogeneous resources for diverse user requirements. In such environments, a portion of resource-intensive tasks, such as large model training jobs, may coexist with standard computational workloads. 
 For the main benchmark settings, tasks are scheduled on three heterogeneous resource pools. 
 The heterogeneity substantially increases the problem size. Additional dataset details and size statistics are provided in Appendix \ref{apx:dataset}.

\textbf{Baselines.}
We compare our method with the following baselines. i) \textit{Heuristic baselines:} list scheduling algorithms, including critical path (CP), shortest first task (SFT), and most operations remaining (MOPNR); Tetris \cite{tetris}, a dynamic list scheduling heuristic for multi-resource pool scheduling; HEFT \cite{heft}, PEFT~\cite{PEFT} and IPPTS~\cite{IPPTS}, heterogeneous scheduling algorithms; For the first four list scheduling algorithms, we apply three pool-selection rules and select the one with the best makespan. ii) \textit{Neural baselines:} PPO-BiHyb \cite{Wang21nips}, a bi-level neural scheduler with beam search that involves multi-round neural network processing, and One-Shot \cite{jeon2023iclr}, a one-shot neural scheduler that generates schedules sequentially based on list scheduling; these two methods are among the strongest existing learning-based baselines for DAG scheduling. We also include ScheduleNet \cite{Wang2022SchNet} and HGN-Two-Stage \cite{Song2023TwoStage}, two HGNN-based methods for heterogeneous graph learning. Since ScheduleNet and HGN-Two-Stage are not designed for DAG scheduling, we adapt them to our setting. The implementation details for these baselines are available in Appendix \ref{apx:baselines}.    

Our method is evaluated in two modes: greedy, which selects actions with the highest probability $p_\theta$ and sampling (S($n$)), which generates $n$ samples based on $p_\theta$ within our accelerated environment and selects the schedule with the minimum makespan. We report the makespan, running time or relative improvement over the best heuristic baseline, with further experimental details provided in Appendix \ref{apx:add_results} and \ref{apx:implementation_details}.

\subsection{Main Benchmark Results}
 On the TPC-H dataset, WeCAN demonstrates up to 11.7\% makespan improvement over the best heuristic and 7.3\% over the best neural baseline, with superior performance in instances with 300$\sim$500 nodes and robust results for 1,000 nodes (see Table~\ref{table:tpch-main}). Additionally, our results show that WeCAN excels at learning robust scheduling policies in complex heterogeneous environments, as evidenced by its performance across diverse TPC-H instances. Owing to single-pass neural processing, WeCAN-greedy achieves lower running time than PPO-BiHyb and comparable running time to One-Shot-greedy, ScheduleNet, HGN-Two-Stage and heuristic baselines, while delivering superior makespan. This similarity in running time arises because, in heterogeneous environments, the generation map's runtime dominates for both WeCAN and One-Shot, approaching the minimum time required to generate a schedule. 
 
 On the computation graph dataset, WeCAN demonstrates up to 13.5\% makespan improvement over the best heuristic and 9.5\% over the best neural baseline, with superior performance across different types of graphs (see Table~\ref{table:ali-main}). Given the prevalence of heterogeneous resource environments in ML compilers, our results demonstrate WeCAN's applicability for efficient scheduling in neural network compilation. 
 
 On the RI-W dataset, WeCAN achieves the best makespan among all compared methods on both RI-W-30 and RI-W-50. It improves the makespan by up to 6.2\% over the best heuristic baseline. Compared with learning-based baselines, WeCAN achieves up to 3.3\% improvement over PPO-BiHyb while requiring more than 10$\times$ less inference time, and up to 10.9\% improvement over OneShot with comparable inference time. These results demonstrate that WeCAN is effective for resource-intensive workloads, especially when strong resource contention makes scheduling decisions more sensitive.

\subsection{Ablation on Skip Design}\label{subsec:ablation-skip}
\begin{figure*}
  \centering
  \includegraphics[width=0.95\linewidth]{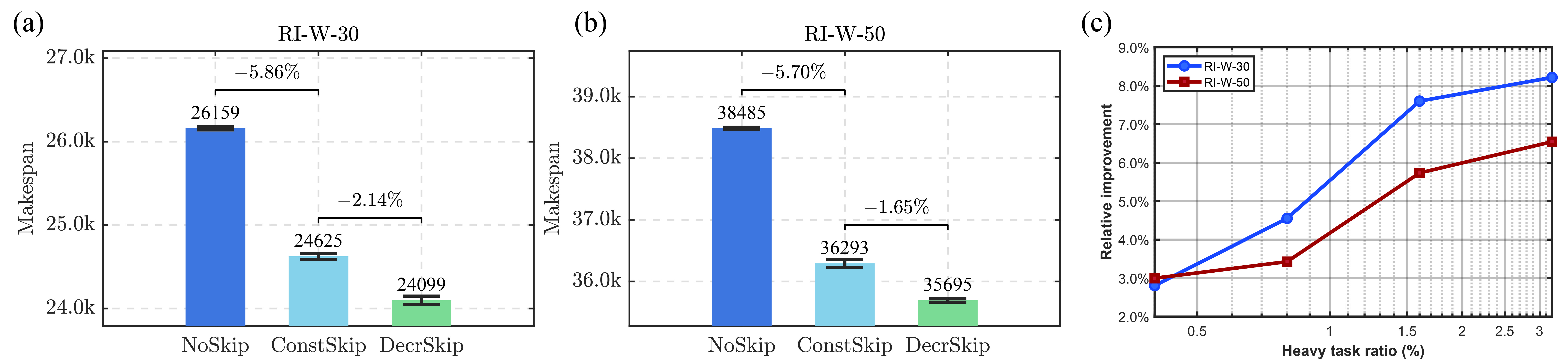}
  \caption{Ablation results on the skip design. (a-b) Results of different skip action score formulations on RI-W-30/50 datasets. ``NoSkip'', ``ConstSkip'' and ``DecrSkip'' correspond to the variants without skip action, with constant-score skip action and with our decreasing-score skip action, respectively. Error bars denote standard deviation. (c) Relative improvement from enabling skip under different heavy-task ratios. }
  \label{fig:skip-abl} 
\end{figure*}

 Across the main benchmark results, timeline-insertion or SGS-like heuristics such as HEFT and IPPTS often achieve smaller makespan than immediate list-scheduling heuristics. This observation is consistent with our order-space analysis that immediate list scheduling can be restrictive under cumulative resource constraints. Motivated by this phenomenon, we conduct ablation experiments on the skip formulation using RI-W-30 and RI-W-50, where resource-intensive tasks amplify such restrictions. The results are presented in  Fig.~\ref{fig:skip-abl} (a-b). WeCAN with a decreasing skip score achieves lower makespan than its non-skip variant, its constant-skip variant, and all competing baselines. These findings verify the effectiveness of the skip action and the necessity of its decreasing formulation.

A non-decreasing skip score does not guarantee order coverage. When $u_{\text{skip}}(k)$ is non-decreasing, once some ready task has a score below $u_{\text{skip}}(k)$, it can remain suppressed until the skip action is masked, which systematically induces unnecessary idling and may exclude desirable orders. The counterexample instance $\mathcal P_0$ (Fig.~\ref{fig:counterex}(a)) illustrates this failure mode. Even though a skip is introduced to enable the optimal schedule, a non-decreasing skip score can still force additional idle periods and yield a worse schedule (Fig.~\ref{fig:counter-fixskip}), showing that the extended image may still miss the optimum and thus the optimality gap is not consistently eliminated.

To further validate the skip design across varying workload intensities, we study the performance of skip action in RI-W datasets with different ratios of heavy tasks. The results in Fig. \ref{fig:skip-abl} (c) reveal that the skip action benefits more as the percentage of heavy tasks increases. This correlation reveals that the optimality gap is positively correlated with the ratio of heavy tasks. Consequently, skip-extended scheduling is especially important when the ratio of heavy tasks is notably large. 

\begin{figure}[htbp]
    \centering
    \includegraphics[width=0.95\linewidth]{figures/pic-fix-skip.jpg}
    \caption{For the counterexample instance $P_0$, a constant-score skip extension still misses the optimum and cannot consistently eliminate the optimality gap.}
    \label{fig:counter-fixskip}
\end{figure}

\begin{table}[t]
\caption{Ablation study results on TPC-H datasets with different architectures.}
\label{table:components}
\centering
\footnotesize
\setlength{\tabcolsep}{1.2pt}
\renewcommand{\arraystretch}{0.95}
\begin{tabular}{@{}>{\raggedright\arraybackslash}p{0.35\columnwidth}rrrr@{}}
\toprule
& \multicolumn{2}{c}{TPC-H-30} & \multicolumn{2}{c}{TPC-H-50} \\
\cmidrule(lr){2-3}\cmidrule(l){4-5}
Architecture
& \multicolumn{1}{c}{Makespan}
& \multicolumn{1}{c}{Improvement}
& \multicolumn{1}{c}{Makespan}
& \multicolumn{1}{c}{Improvement} \\
\midrule
IPPTS & 21389 & 0.0\% & 37164 & 0.0\% \\
WeCA + LDDGNN & \textbf{19908} & \textbf{6.92\%} & \textbf{34260} & \textbf{7.81\%} \\
WeCA-inside + LDDGNN & 20729 & 3.09\% & 34980 & 5.88\% \\
WeCA-alt + LDDGNN & 20234 & 5.40\% & 34815 & 6.32\% \\
WeCA-inside-alt + LDDGNN & 21981 & -2.77\% & 36984 & 0.48\% \\
No-WeCA + LDDGNN & 23066 & -7.84\% & 40308 & -8.46\% \\
WeCA + GAT (forward) & 20747 & 3.00\% & 35224 & 5.22\% \\
WeCA + GAT (bidirection) & 20873 & 2.41\% & 35177 & 5.35\% \\
\bottomrule
\end{tabular}
\end{table}
    
\subsection{Ablation on WeCAN Architecture}
We conduct ablation experiments to assess the contributions of the WeCA layers and the LDDGNN module in WeCAN. For WeCA layers, we test the inside variant (WeCA-inside), removing the initial task-only WeCA layer while keeping the alternating updates (WeCA-alt), its inside counterpart (WeCA-inside-alt), and removing all WeCA layers (No-WeCA). For LDDGNN, we compare a standard GAT (GAT forward) and a bidirectional GAT (GAT bidirection). All network variants share the same layer count and hidden dimensions, with fewer WeCA layers offset by additional LDDGNN layers. We train and test each model on TPC-H-30 and TPC-H-50. For each of 10 test problems, we generate 256 samples, computing mean makespan and relative improvement to the best heuristic, IPPTS. Results in Table~\ref{table:components} show that modifying or removing either component increases the makespan. In particular, removing WeCA layers leads to a clear performance degradation, which highlights the important role of WeCA in modeling task and pool interactions. Both GAT variants yield higher makespan than LDDGNN, demonstrating the superiority of LDDGNN.

\subsection{Learning under Different Generation Maps}

\begin{figure*}[htbp]
    \centering
    \includegraphics[width=0.8\textwidth]{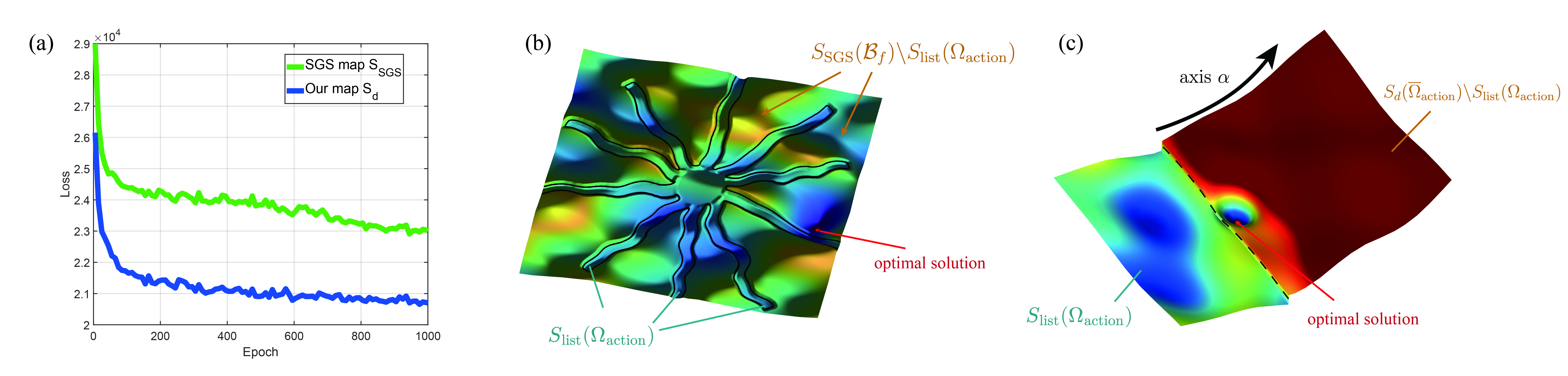}
    \caption{Training results with different generation maps and their explanation by toy illustration. (a) Training loss on the TPC-H-30 dataset, serial schedule generation scheme $S_{\text{SGS}}$ exhibits significantly slower convergence and inferior performance compared to our generation map $S_d$. (b) The toy illustration for landscape by $S_{\text{SGS}}$. The flat basins inside black lines represent the image of list scheduling. The extension introduces extra regions containing the optimal solution and the numerous extra mountains corresponding to worse solutions. This irregular mixed structure slows down the convergence. (c) The toy illustration for landscape by $S_d$. The extended spaces brings worse solutions but the worst solutions are mostly far away in the axis of $\alpha$ while the optimal solutions are usually near the boarder line. This regular structure stabilizes and accelerates the convergence.}
    \label{fig:learning-with-generation-map}
    \vspace{-5pt}
\end{figure*}
Training curves under the serial SGS $S_{\text{SGS}}$ and our generation map $S_d$ on the TPC-H-30 benchmark are depicted in Fig.~\ref{fig:learning-with-generation-map} (a). It can be seen that  $S_{\text{SGS}}$ converges noticeably slower and reaches a worse final loss level. 

\textit{Why $S_{\text{SGS}}$ slows down learning}.
Replacing list scheduling $S_{\text{list}}$ with  $S_{\text{SGS}}$ enlarges the image by admitting more feasible schedules, but this expansion is too unconstrained for learning. To preserve the schedule order, $S_{\text{SGS}}$ permits arbitrary idle time before launching the next task when resources are not yet sufficient. As a result, many inferior schedules are introduced across fibers in an irregular manner (Fig.~\ref{fig:learning-with-generation-map} (b)), which makes policy optimization more difficult and slows down convergence.

\textit{Why $S_d$ is a milder extension.}
In contrast, the poor schedules introduced by $S_d$ predominantly arise from excessive skips, which correspond to large $(\alpha,\beta)$ in the skip-score formula. Hence, the undesirable region is more structured and concentrated, rather than being scattered throughout the feasible order space as in $S_{\text{SGS}}$. The toy landscape illustration in Fig.~\ref{fig:learning-with-generation-map} (c) visualizes this distinction. $S_d$ enlarges the image while keeping the worst solutions far away along the skip-parameter axes, so learning can focus on the reachable set of $S_{\text{list}}$ and its nearby boundary where performance improvements are most likely to occur. This design retains the ability to include optimal schedules while avoiding a substantial slowdown in training.

\subsection{Generalization Results}

\begin{figure*}[htbp]
  \centering
  \includegraphics[width=0.95\linewidth]{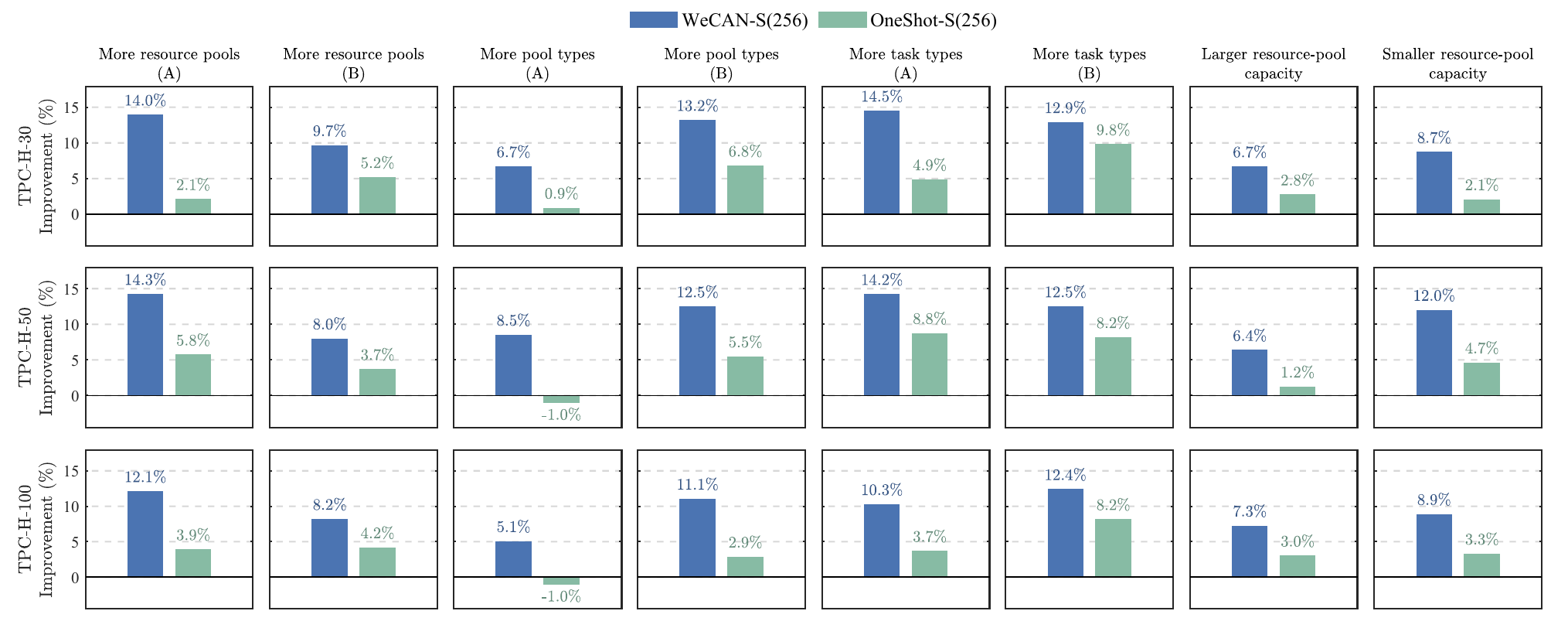}
  \caption{Evaluations of models on TPC-H with different environment fluctuations under fixed training conditions. The percent of improvement over best heuristics are labeled.}
  \label{fig:gen}  
\end{figure*}

To evaluate the cross-size generalization ability of WeCAN to large-scale problems, we present the experimental results on TPC-H-150 ($\approx$ 1500 tasks) and TPC-H-200 ($\approx$ 2000 tasks) in Table~\ref{table:large}. The model is trained on TPC-H-30 with the original 3 pool environment.  Our results show that our WeCAN still outperforms the best heuristics by about $7\%-10\%$ and the best learning methods by about $5\%$, while remaining comparable running time in greedy mode. These results indicate that WeCAN scales favorably to large scale problems.

Furthermore, we conduct experiments in varying resource environments to evaluate generalization from a fixed training environment. Fig.~\ref{fig:gen} shows that our WeCAN remains robust under varying fluctuations in the number and capacities of resource pools as well as the pool and task types that dictate the compatibility coefficients. WeCAN consistently outperforms OneShot and heuristic baselines across all evaluation scenarios. This validates that WeCAN effectively utilizes the environment feature while remaining robust and scalable for environment sizes. Details of the result can be found in Appendix \ref{apx:add_results}.

\begin{table*}[htbp]
\caption{Full generalization results on large-scale TPC-H datasets}
\label{table:large}
\setlength{\tabcolsep}{2.8pt}
\centering
\small
\resizebox{\textwidth}{!}{%
\begin{tabular}{l*{12}{c}}
\toprule
 & \multicolumn{6}{c}{3 pools} & \multicolumn{6}{c}{12 pools} \\
\cmidrule(lr){2-7}\cmidrule(lr){8-13}
 & \multicolumn{3}{c}{TPC-H-150} & \multicolumn{3}{c}{TPC-H-200}
 & \multicolumn{3}{c}{TPC-H-150} & \multicolumn{3}{c}{TPC-H-200} \\
\cmidrule(lr){2-4}\cmidrule(lr){5-7}\cmidrule(lr){8-10}\cmidrule(lr){11-13}
 & Makespan & Improvement & Time
 & Makespan & Improvement & Time
 & Makespan & Improvement & Time
 & Makespan & Improvement & Time \\
\midrule
Tetris & 109079 & -2.54\% & 3.65 & 140438 & -2.52\% & 6.09 & 30552 & -16.54\% & 4.80 & 37737 & -11.18\% & 8.51 \\
HEFT & 108285 & -1.80\% & 3.26 & 139094 & -1.53\% & 4.98 & 27415 & -4.58\% & 3.11 & 34834 & -2.63\% & 5.21 \\
PEFT & 108490 & -1.99\% & 3.05 & 140154 & -2.31\% & 5.04 & 27943 & -6.59\% & 3.60 & 35539 & -4.71\% & 4.98 \\
IPPTS & 106372 & 0.00\% & 2.74 & 136992 & 0.00\% & 4.80 & 26215 & 0.00\% & 5.75 & 33942 & 0.00\% & 4.88 \\
PPO-BiHyb & 102770 & 3.39\% & 447.15 & 134341 & 1.94\% & 740.21 & 25607 & 2.32\% & 465.01 & 33358 & 1.72\% & 700.26 \\
One-Shot-S(256) & 101690  & 4.40 \%  & 15.10 & 131941  & 3.69 \%  & 23.50 & 25679  & 2.04 \%  & 16.83 & 33016 & 2.73 \% & 26.15 \\
WeCAN-Greedy & 96225 & 9.54\% & 2.65 & 124345 & 9.23\% & 4.97 & 24761 & 5.55\% & 3.48 & 31864 & 6.12\% & 6.09 \\
WeCAN-S(256) & \textbf{95491 } & \textbf{10.23 \%} & 19.10 & \textbf{124140 } & \textbf{9.38 \%} & 29.70 & \textbf{24338 } & \textbf{7.16 \%} & \textbf{18.36} & \textbf{31422 } & \textbf{7.42 \%} & \textbf{30.40} \\
\bottomrule
\end{tabular}%
}
\end{table*}
\section{Conclusion}
In this paper, we presented WeCAN, an end-to-end reinforcement learning framework for heterogeneous DAG scheduling with task-pool compatibility coefficients under the single-pass inference regime. 
We also developed an order-space analysis that characterizes the restricted reachable set induced by a generation map and formalizes the resulting generation-induced optimality gap. The analysis further provides explicit sufficient conditions for gap elimination. 
Guided by this principle, WeCAN combines a size-agnostic weighted cross-attention encoder modeling compatibility-gated task-pool interactions, and a skip-extended generation map using an analytically parameterized decreasing skip rule. The generation map improves order reachability while preserving single-pass efficiency without repeated network calls.
Empirical results on TPC-H query settings, RI-W workload, and ML-compiler computation graphs demonstrate that WeCAN improves makespan over baselines with inference time comparable to classical heuristics.
An interesting direction for future work is to extend the same order-space analysis and single-pass realization principles to more complex scenarios, such as communication- or data-movement-aware scheduling, stochastic processing times, or dynamic workload arrivals.

{\footnotesize
\bibliographystyle{IEEEtran}
\bibliography{reference}
}

\newpage

\clearpage
\onecolumn

\appendices
\renewcommand{\thesubsection}{\thesection.\arabic{subsection}}
\renewcommand{\thesubsectiondis}{\thesubsection}

\renewcommand{\thesubsubsection}{\thesubsection.\arabic{subsubsection}}
\renewcommand{\thesubsubsectiondis}{\thesubsubsection}



\section{Supplement theoretical results}\label{apx:theoritical results}

\setcounter{figure}{7}
\setcounter{table}{5}
\setcounter{theorem}{1}
\setcounter{proposition}{5}
\setcounter{equation}{8}
\setcounter{algorithm}{2}

\subsection{MILP formulations of the scheduling problem}

A heterogeneous scheduling problem $\mathcal P=(V,E,\mathcal C,t,\rho,\lambda,K_{acc})$ can be formulated as
\begin{equation}
\label{eq:he-prob-concise}
\begin{aligned}
\min_{x=(s(\cdot), c(\cdot))}\quad&
f(x)=\max_{v\in V}\bigl[ s(v)+t_{\mathrm{act}}(v, c(v))\bigr],\\
\mathrm{s.t.}\qquad&
 s(v)+t_{\mathrm{act}}(v, c(v))\le  s(w),\quad \forall (v,w)\in E,\\
&
\sum_{v\in F(\tau,c)} \rho(v)\le \lambda(c),\quad \forall \tau\ge 0,\ \forall c\in\mathcal C,\\
&
K_{acc}(v, c(v))>0,\quad \forall v\in V,\\
&
 s(v)\ge 0,\quad \forall v\in V.
\end{aligned}
\end{equation}
Here, $F(\tau,c)$ is defined as 
\begin{equation*}
    F(\tau,c):=\bigl\{v\in V\mid  c(v)=c,\, s(v)\leq\tau<t_{\mathrm{act}}(v, c(v))\bigr\}.
\end{equation*}

For the homogeneous scheduling problem $(V,E,t,\rho,\lambda)$ with $r$ types of resources, where $K\equiv 1$ and $|C|=1$, the third constraint is naturally satisfied. $V$ is the task set with $|V|=n$ and $E$ is the edge set. Let constant $t_i$ be the processing time of task $i$, $\rho_l^i$ be the $l$-th type resource demand of task $i$, and $\lambda_l$ be the capacity of the $l$-th type resource. Let constant $C_1^{\text{up}}$ be a sufficiently large number and can be set to $\sum_{i\in V} t_i$. Let variable $s_i$ denote the start time of task $i$, variable $t_{max}$ denote the makespan. Let variable $u_{ij}$ equal $0$ if and only if task $i$ starts before task $j$ starts, variable $w_{ij}$ equal $0$ if and only if task $i$ starts after or at the same time as task $j$ finishes, variable $x_{ij}$ equal $1$ if and only if task $j$ is running when task $i$ starts. The corresponding mixed integer linear programming (MILP) formulation is given by (\ref{eq:ho-milp}).

\begin{equation}
\label{eq:ho-milp}
    \begin{aligned}
    \min\;& t_{max}\\
    \st\; &  s_i+t_i\le t_{max}, \forall i \in V,\\
         & s_i+t_i\le s_j, \forall (i,j) \in E,\\
         & s_i < s_j+C_1^{\text{up}}u_{ij}, \forall i,j \in V,\\
         &s_j \le s_i+C_1^{\text{up}}(1-u_{ij}), \forall i,j\in V,\\
         & s_i < s_j +t_j + C_1^{\text{up}}(1-w_{ij}), \forall i,j\in V,\\
         & s_j+t_j \le s_i + C_1^{\text{up}}w_{ij}, \forall i,j\in V,\\
         & u_{ij}+w_{ij}\le 1 + x_{ij}, \forall i,j\in V, \\
         & u_{ij}+w_{ij}\ge  2x_{ij}, \forall i,j\in V,\\
         & \rho_l^i+ \sum_{j\ne i} x_{ij} \rho_l^j \le \lambda_l, \forall i \in V,l \in \{1,...,r\},\\
         & s_i\ge 0, t_{max}\ge 0,\, u_{ij},w_{ij},x_{ij} \in \{0,1\}, \forall i,j\in V.
    \end{aligned}
\end{equation}

For the heterogeneous scheduling problem $\mathcal P = (V,E,\mathcal C,t,\rho,\lambda, K_{acc})$ with $r$ types of resources, where $\mathcal C$ is the pool set and $K_{acc}$ is the compatibility coefficient matrix. Let variables $s_i, t_{max}, u_{ij}, w_{ij}, x_{ij}$ and constants $t_i, \rho_l^i$ share the same meaning as in the homogeneous case. Let variable $v_i^k$ equal $1$ if and only if task $i$ is assigned to pool $k$, variable $y_{ijk}$ equal $1$ if and only if task $j$ is running on pool $k$ when task $i$ starts. Let constant $C_0^{\text{up}}$ be a sufficiently large bound and can be set to $2\max\{1/K_{acc}(i,k): K_{acc}(i,k) \ne 0,\,\forall i,k\}$. Let constant $a_{ik} = \min\{1/K_{acc}(i,k), C_0^{\text{up}}\}$, constant $C_1^{\text{up}}$ be a sufficiently large number and can be set to $\sum_{i\in V, k\in C}t_ia_{ik} + \sum_{i\in V, l =1...,r}\rho_l^i$. Let constant $\lambda_l^k$ be the $l$-th type resource capacity of pool $k$. The MILP formulation of the heterogeneous scheduling problem is formulated by (\ref{eq:he-milp}). The MILP formulation for heterogeneous scheduling involves substantially more intricate constraints and a greater number of variables and constraints compared to the homogeneous case, significantly increasing the difficulty of solving the problem.
\begin{equation}
\label{eq:he-milp}
    \begin{aligned}
    \min\;& t_{max},\\
    \st\; & s_i+t_i\sum_{k} v_i^k a_{ik} \le t_{max}, \forall i\in V,\\
         & s_i+t_i\sum_{k\in C} v_i^k a_{ik}\le s_j, \forall (i,j) \in E,\\
         & s_i < s_j+C_1^{\text{up}}u_{ij}, \forall i,j\in V,\\
         &s_j \le s_i+C_1^{\text{up}}(1-u_{ij}), \forall i,j\in V,\\
         & s_i < s_j +t_j\sum_{k\in C} v_j^k a_{jk} + C_1^{\text{up}}(1-w_{ij}), \forall i,j\in V,\\
         & s_j+t_j\sum_{k\in C} v_j^k a_{jk} \le s_i + C_1^{\text{up}}w_{ij}, \forall i,j\in V,\\
         & u_{ij}+w_{ij}\le 1 + x_{ij}, \forall i,j\in V, \\
         & u_{ij}+w_{ij}\ge  2x_{ij}, \forall i,j\in V,\\
         & x_{ij}+v_j^k\ge 2y_{ijk}, \forall i,j\in V,k\in C, \\
         & x_{ij}+v_j^k\le 1+y_{ijk}, \forall i,j\in V,k\in C,\\
         & \sum_k v_i^k = 1, \forall i\in V,k\in C,\\ 
         & 1 - v_i^k + C_0^{\text{up}} - a_{ik} > 0,\, \forall i\in V,k\in C,\\
         & \rho_l^i+ \sum_{j\ne i} y_{ijk} \rho_l^j \le C_1^{\text{up}}(1-v_i^k) +\lambda_l^k, \forall i\in V,k\in C,l\in\{1,...,r\},\\
         & s_i\ge 0 ,t_{max}\ge 0,\, u_{ij},w_{ij},x_{ij},y_{ijk},v_i^k \in \{0,1\}, \forall i,j\in V,k\in C.
    \end{aligned}
\end{equation}

 For a given solution $z = \{s_i, t_{max}, u_{ij},w_{ij},x_{ij},y_{ijk},v_i^k\}_{i,j,k}$, we define \begin{equation*}
     \phi(z) = x,\, \text{where } x(i) = (s(i), c(i)) =(s_i,\{k:v_i^k = 1\}).
 \end{equation*} 
 By construction, the map $\phi$ takes any feasible MILP solution as input and outputs a schedule in the schedule space $\mathcal X$ (defined in the main text). Due to the definition of the variables, we immediately have that $\phi$ builds a one-to-one correspondence between the space of all feasible MILP solutions and $\mathcal X$, which further reveals that the MILP formulation \eqref{eq:he-milp} is equivalent to the problem \eqref{eq:he-prob-concise}.
\begin{proposition}
    The map $\phi$ is a bijection from the space of all feasible MILP solutions to $\mathcal X$.
\end{proposition}
\begin{proof}
    Due to the definition of variables, the rest variables can be inferred from $s(\cdot)$ and $c(\cdot)$. Therefore, $\phi$ is bijection and surjection.
\end{proof}

Although the map $\phi$ theoretically establishes a one-to-one correspondence between the solution spaces, the MILP formulation \eqref{eq:he-milp} exhibits significantly higher computational complexity than our formulation \eqref{eq:he-prob-concise}. The MILP approach necessitates the explicit linearization of disjunctive constraints regarding resource capacities and task ordering, which requires introducing numerous auxiliary binary variables. Consequently, the problem size scales quadratically ($O(n^2)$) in terms of both variables and constraints. In contrast, our formulation \eqref{eq:he-prob-concise} naturally characterizes the scheduling problem using only start times and pool assignments, maintaining a linear scale ($O(n)$) in the solution space. To quantitatively illustrate this advantage, we summarize the problem size complexity in Table~\ref{tab:complexity_comp}, which highlights the substantial reduction in dimensionality achieved by our formulation. Problem sizes of these different formulations  on specific benchmarks can be seen in Table~\ref{tab:pb_sz} of Appendix \ref{apx:dataset}.3. 

\begin{table}[htbp]
\centering
\caption{Complexity comparison of problem formulations}
\label{tab:complexity_comp}
\renewcommand{\arraystretch}{1.2}
\begin{tabular}{llcc}
\toprule
Setting & Formulation & Variables & Constraints \\
\midrule
\multirow{2}{*}{Homogeneous} & Problem \eqref{eq:he-prob-concise} & $n$ & $|E| + O(n)$ \\
& MILP \eqref{eq:ho-milp} & $3n^2$ & $6n^2$ \\
\midrule
\multirow{2}{*}{Heterogeneous} & Problem \eqref{eq:he-prob-concise} & $2n$ & $|E| + O(n)$ \\
& MILP \eqref{eq:he-milp} & $(3+n_c)n^2$ & $(6+2n_c)n^2$ \\
\bottomrule
\end{tabular}
\end{table}

\subsection{Details of generation map $S_d$}

As mentioned in Section \ref{sec:generating-feasible-schedule} of the main text, Algorithm 2 induces a skip-extended generation map $S_d$, whose domain is first extended from $\Omega_\text{action}$ by the extended action set $\bar{\mathcal A}$. The domain is defined as follows:

\begin{equation}
\overline{\Omega}_{\text{action}} = \left\{ (a_1, \dots, a_N) \;\middle|\; 
\begin{aligned}
    & a_i \in \bar{\mathcal{A}}, \, \forall i \in \{1, \dots, N\}, \\
    & \#\{i:a_i\in \mathcal A\} = n, \\
    & (a_i=(v_i,\overline c_i), a_j=(v_j,\overline c_j))\Rightarrow v_i\ne v_j,\, \forall i\ne j
\end{aligned}
\right\}.
\end{equation}
The induced map $S_d:\overline\Omega_\text{action}\to \mathcal X$ is formulated as Algorithm \ref{alg:S_d}.

\begin{algorithm}[htb]
   \caption{Induced skip-extended generation map $S_d$}
   \label{alg:S_d}
\begin{algorithmic}[1]
   \STATE {\bfseries Input:} Instance $\mathcal P=(V,E,\mathcal C,\rho,t,\lambda,K_{\mathrm{acc}})$; extended action sequence $(a_1,...,a_N) \in \overline \Omega_\text{action}$.
   \STATE {\bfseries Output:} A feasible schedule $x=( s(\cdot), c(\cdot))$.
   \STATE Initialize current time $t_{\mathrm{cur}}\gets 0$, unscheduled tasks $U\gets V$, unselected action indices $I\gets \{1,...,N\}$ and pool states (available resources and running tasks).
   \WHILE{$U\neq \emptyset$}
      \STATE Compute feasibility masks for $\mathcal A$ according to dependency/resource/compatibility constraints at $t_{\mathrm{cur}}$.
      \IF{All action in $\mathcal A$ are masked}
          \STATE Advance time to the next completion event and release all tasks completing at the new $t_{\text{cur}}$.
          \STATE Update dependency status and pool states.
          \STATE \textbf{continue}
      \ENDIF
      
      \STATE Find the smallest unmasked action indices $t$ in $I$.
      
      \IF{$a_t=(v,c)\in\mathcal A$}
          \STATE Start task $v$ on pool $c$ at time $t_{\mathrm{cur}}$: set $ s(v)\gets t_{\mathrm{cur}}$, $ c(v)\gets c$.
          \STATE Update pool state and remove $v$ from $U$.
      \ELSE
          \STATE Advance time to the next completion event and release all tasks completing at the new $t_{\text{cur}}$.
      \ENDIF
      \STATE Update dependency status and pool states.
      \STATE Remove $t$ from $I$.
   \ENDWHILE
\end{algorithmic}
\end{algorithm}

\subsection{Trials of local search}
\label{apx:role-ls}

Although the serial schedule generation scheme  \( S_{\text{SGS}} \) cannot be  used directly for learning due to it slowing down the training, it enables an alternative two-step approach: sampling and refinement, with the latter realized by local search. In the first step, a sampler generates suboptimal feasible schedules, and in the second step, local search identifies potential optimal schedules in their neighborhoods. This approach mitigates optimality gaps in \( S_{\text{list}} \) by leveraging \( S_{\text{SGS}} \) in local search, aligning with WeCAN's goal of effectively and gap-aware scheduling.

To illustrate this process, we propose a local search strategy using the serial schedule generation scheme \( S_{\text{SGS}} \). For a feasible schedule \( x \in \mathcal X \) from the sampler, the natural projection \( T \) yields its schedule order \( T(x) \in \mathcal B_f \). The neighborhood \( N(T(x)) \subset \mathcal B_f \) comprises feasible schedule orders obtained by inserting a task into another position in a pool. The local search algorithm selects the optimal schedule in \( S_{\text{SGS}}(N(T(x))) \), completing one search step. Theorem \ref{thm3} guarantees the connectivity of \( \mathcal B_f \), ensuring any feasible schedule can reach an optimal schedule through finite task insertions.

\begin{theorem}
\label{thm3}
For any feasible schedule $x\in\mathcal X$, there exists an optimal schedule $x^\star\in\mathcal X$
and a finite sequence of feasible schedules $x^{(0)},x^{(1)},\dots,x^{(m)}\in\mathcal X$
with $x^{(0)}=x$ and $x^{(m)}=x^\star$ such that
\[
T\!\left(x^{(i+1)}\right)\in N\!\left(T\!\left(x^{(i)}\right)\right),\qquad i=0,\dots,m-1.
\]
\end{theorem}

\begin{proof}
It suffices to prove that the feasible order space $\mathcal B_f:=T(\mathcal X)$ is connected under insertion moves,
and then lift an order path to schedules via $S_{\mathrm{SGS}}$.

\textit{Step 1 (connectivity of $\mathcal B_f$).}
We prove by induction on $n=|V|$ that any two orders $w,w'\in\mathcal B_f$ are connected by finitely many insertions.
The case $n=1$ is immediate.

Assume the claim holds for $n-1$ tasks. Fix any $w,w'\in\mathcal B_f$ on $V$.
Since $(V,E)$ is a DAG, there exists a \emph{sink} task $v_0$ with no successors, i.e., no edge $(v_0,u)\in E$.
From $w$ (resp.\ $w'$), perform one insertion that moves $v_0$ to the end of its current pool,
obtaining $\tilde w\in N(w)$ (resp.\ $\tilde w'\in N(w')$).
This move keeps all pool assignments unchanged, preserving compatibility/capacity feasibility; moreover,
moving a sink task later in its pool cannot create a directed cycle in the augmented graph,
so $\tilde w,\tilde w'\in\mathcal B_f$.

Remove $v_0$ from $\tilde w,\tilde w'$ and relabel ranks within each pool, yielding
$\bar w,\bar w'\in\mathcal B_f^{(n-1)}$ for the reduced instance on $V\setminus\{v_0\}$.
By the induction hypothesis, there exists a sequence
\[
\bar w=\bar w^{(0)},\bar w^{(1)},\dots,\bar w^{(L)}=\bar w',
\qquad \bar w^{(\ell+1)}\in N(\bar w^{(\ell)}).
\]
For each $\ell$, add $v_0$ back to the end of the same pool as in $\tilde w$, producing $w^{(\ell)}\in\mathcal B_f$.
Then $w^{(\ell+1)}\in N(w^{(\ell)})$ holds because the insertion acts only on tasks in $V\setminus\{v_0\}$.
Concatenating $w\to\tilde w=w^{(0)}\to\cdots\to w^{(L)}=\tilde w'\to w'$ gives a finite insertion path in $\mathcal B_f$.

\textit{Step 2 (lift to schedules).}
Let $w^{(0)},\dots,w^{(M)}$ be an insertion path in $\mathcal B_f$ from $T(x)$ to $T(x^\star)$.
Define $x^{(i)}:=S_{\mathrm{SGS}}(w^{(i)})$. By Proposition~3 each $x^{(i)}$ is feasible and
\[
T\!\left(x^{(i)}\right)=w^{(i)}.
\]
Hence $T(x^{(i+1)})\in N(T(x^{(i)}))$ for all $i$, and $x^{(M)}$ is optimal, completing the proof.
\end{proof}

This local search approach with projection iteratively improves schedules by exploring the original space piece by piece, with Theorem \ref{thm3} guaranteeing its feasibility in approaching an optimal solution. Figure \ref{fig:a4} illustrates how local search with the serial schedule generation scheme \( S_{\text{SGS}} \) also fixes the inherent optimality gap in list scheduling. Panel (a) shows that the optimum is excluded by the list scheduling map $S_\text{list}$ through mapping it to a schedule with a different order. Panel (b) shows that although $S_\text{list}$ excludes the optimal solution, applying the local search process on $T(S_\text{list}(\omega))$ and then performing the serial schedule generation scheme map $S_\text{SGS}$ may also helps to obtain the optimal solution $S_{\text{SGS}}(LS(T(S_\text{list}(\omega))))$. This provides an alternative way for overcoming the optimality gaps. Experimental results validating the effectiveness of local search are shown in Appendix~\ref{apx:results-ls}.

\begin{figure}
  \centering
  \includegraphics[width=1.0\linewidth]{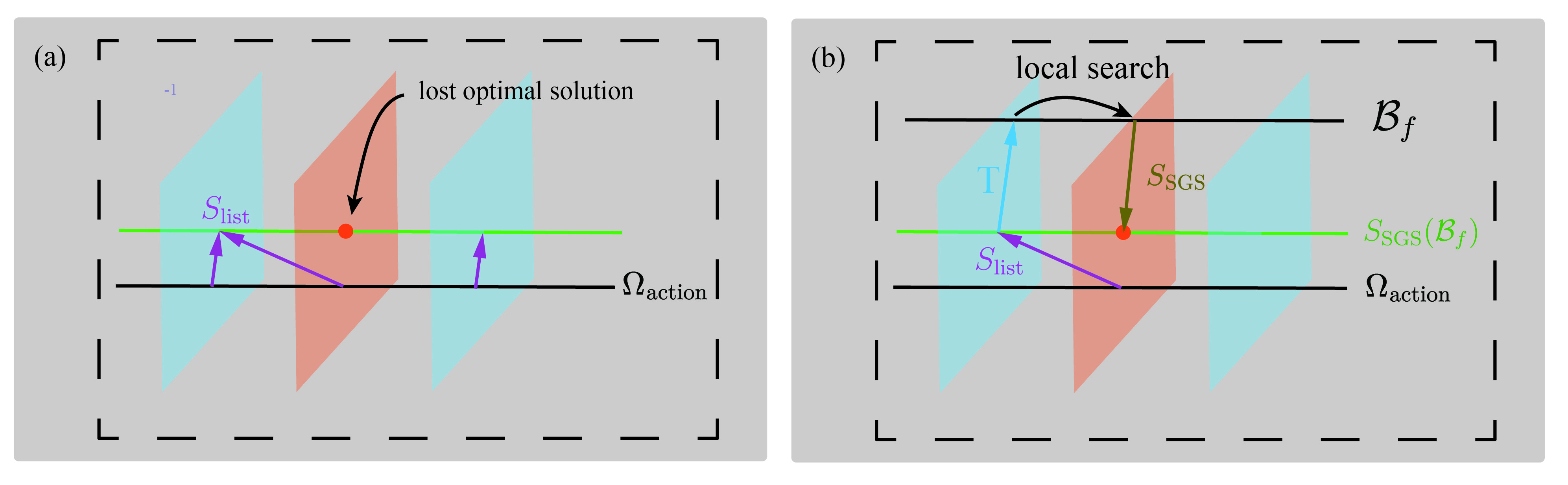}
  \caption{
  Local search with projection. (a) The optimal schedule is excluded by the list scheduling map \( S_{\text{list}} \) due to a different schedule order. (b) Local search \( LS \) on the schedule order \( T(S_{\text{list}}(\omega)) \), followed by the serial schedule generation scheme  \( S_{\text{SGS}} \), recovers the optimal schedule \( S_n(LS(T(S_{\text{list}}(\omega)))) \).}
  \label{fig:a4}
\end{figure}

\section{Benefit from single-pass network framework}\label{apx:narvsar}

Instead of utilizing auto-regressive models, we employ a single-pass (non-auto-regressive) model because it takes the following two benefits and does not bring noticeable performance loss.
\begin{itemize}
    \item \textbf{Computation‐speed gains from non‐auto-regressive decoding:}  
   The single‐pass architecture rapidly generates solutions with just one neural network forward, enabling rapid schedule generation even on CPU‐only servers. As shown in Appendix \ref{apx:inference_time} and Table \ref{table:eva-time}, over 90 \% of inference time is spent in the generation‐map step (which any scheduler must perform), so our total runtime approaches the theoretical lower bound. By contrast, an auto-regressive decoder would rerun the network for each of the $n$ tasks, on TPC-H-30 this inflates total cost by $\approx$ 10× on GPU, and on TPC-H-100 by $\approx$ 100× on CPU, making it orders of magnitude slower.
   \item \textbf{Reduced training GPU memory with non‐autoregressive decoding:}  
   During training, backpropagation must keep all forward‐pass activations in memory which dominates the GPU memory costs of DAG scheduling. In evaluation process, activations can be immediately freed. But in training process, they must be kept after the solution generated and reward calculated. An auto-regressive decoder runs the full network at each of the $n$ scheduling steps, retaining $n$ sets of activations, so you quickly run out of GPU memory and must use tiny batch sizes (e.g. on TPC-H-30/50) and cannot scale to larger DAG scheduling problem. In contrast, our single‐pass design invokes the network exactly once per problem, storing only one set of activations. This lets us train on much larger problems (up to $\approx$ 1000 tasks) with practical batch sizes.

\end{itemize}
  
Unlike combinatorial problems such as TSP where each decision reshapes the remaining solution, in DAG scheduling each choice mainly affects which tasks are running (and which masks are active). Once a task finishes, it no longer influences later decisions beyond updating the dependency mask. Because any selectable task already has all its parents completed, available tasks are independent except for shared resource contention. Intuitively, this matches real-world practice: you plan your next jobs based on currently running workloads, not on jobs that have already finished. Moreover, the following two experiments results shows that non-auto-regressive model does not lead to noticeable performance loss.

\textit{Evidence from local search:}  
   If early decisions had a strong ripple effect, perturbing the action order would yield substantial improvements. We implemented a local search that reinserts any action at any position and keeps the best reorder. As shown in Table~\ref{table:ls} of Appendix \ref{apx:add_results}, 100 local-search steps further improve WeCAN-S(256) by about 2.6\% on RI-W-30 and 1.1\% on RI-W-50, while increasing runtime by nearly two orders of magnitude. This suggests that the single-pass policy already captures most of the useful ordering information once the generation map is properly designed.

 \textit{Minimal gains from auto-regressive models in DAG scheduling:} Before adopting the non-auto-regressive (NAR) model, we start our work by implementing an auto-regressive (AR) version that encodes both current and future resource availability (i.e., running tasks) for each pool. Since pools may have variable numbers of running tasks but embeddings must be fixed-size, we represent each pool’s state with quantile information. Let $K$ be a integer-value hyper-parameters. For pool $c$, $t_{cur}$ denotes the current time, $r_{cur}$ the current resource capacity vector, and $r_f$ the full capacity. At each step, pool $c$ is embedded into $(r+1)\cdot(K+1)$-dimensional dynamical feature vector $(t_1,r_1,...,t_{K+1},r_{K+1})$, where $r$ is the number of resource types. Here, $t_i$ is the earliest time the resource type 1 will be released to reach $r_{cur}^{(1)} + \frac{i-1}{K}(r_f^{(1)}-r_{cur}^{(1)})$ if no more task is coming, and $r_i$ is the corresponding resource vector at $t_i$. Notably, we have $t_1=t_{cur}$, $r_1=r_{cur}$, $t_{K+1}$ is the maximal finish time for tasks running on pool $c$, and $r_{K+1}=r_{ful}$, so this embedding captures the current resource state, the full resource capacity, and critical resource states over time.

The AR model is identical to the NAR model except at the decoder: at each step, the decoder extracts these dynamic features for every pool, processes them with an MLP for pool embeddings, then applies the WCA layer to update both task embeddings and pool–task scores. By taking $K=5$, we evaluated the AR model on the TPC-H dataset, taking the maximal batch size (20 for TPC-H-30 and 8 for TPC-H-50) not out of GPU-memory.
   
\begin{table*}[h!]
\centering
\caption{Performance comparison on TPC-H dataset.}
\label{tab:tpch-results}
\begin{tabular}{lcccccc}
\toprule
 & \multicolumn{2}{c}{TPC-H-30} & \multicolumn{2}{c}{TPC-H-50} & \multicolumn{2}{c}{TPC-H-100} \\
\cmidrule(lr){2-3} \cmidrule(lr){4-5} \cmidrule(lr){6-7}
 & \multicolumn{1}{c}{Makespan} & \multicolumn{1}{c}{Time} & \multicolumn{1}{c}{Makespan} & \multicolumn{1}{c}{Time} & \multicolumn{1}{c}{Makespan} & \multicolumn{1}{c}{Time} \\
\midrule
WeCAN-Greedy(NAR)     & 19578 & 0.15 & 33428 & 0.50 & 62587 & 1.72 \\
WeCAN-S(256)(NAR)     & 18964$\pm$10 & 2.43 & 32814$\pm$47 & 2.86 & 61373$\pm$28 & 10.43 \\
WeCAN-AR-Greedy       & 19386 & 1.96 & 33786 & 3.39 & (out of GPU-memory) & * \\
WeCAN-AR-S(256)       & 18910$\pm$26 & 3.57 & 33201$\pm$75 & 7.05 & (out of GPU-memory) & * \\
\bottomrule
\end{tabular}
\end{table*}
   
From the results, the AR model achieves only less than 1\% better makespan on TPC-H-30 and even 1\% worse performance on TPC-H-50 (this may due to the small batch size due to larger GPU memory costs), but requires substantially more inference time (more than 10 times of NAR) and training GPU memory. This trade-off is the main reason we adopted the non-auto-regressive model, single-pass design.

\section{Dataset}\label{apx:dataset}
\subsection{TPC-H dataset}

In the TPC-H dataset\footnote{http://tpc.org/tpch/default5.asp}, each DAG contains between 2 and 18 nodes, with $9.18$ nodes on average to be scheduled. The dataset used by Wang et. al.~\cite{Wang21nips}\footnote{https://github.com/Thinklab-SJTU/PPO-BiHyb/tree/main/dag\_data/tpch} only includes one type of resource demand, with an average value of 125.8, a minimum of 1, and a maximum of 593. To better reflect heterogeneous real-world scenarios, we extend this dataset by introducing a second resource dimension (representing memory usage) and assigning a random type to each task to indicate its primary computational load (e.g., CPU, IO, etc.). The second resource demand is sampled uniformly from $\{30, 40, 50\}$, and task types are sampled  uniformly from $\{0, 1, 2\}$. 

We construct training and test datasets under three different problem-size settings, denoted as TPC-H-30, TPC-H-50, and TPC-H-100. Under each setting, a problem instance consists of 30, 50, or 100 DAGs, respectively. The training datasets under these three settings contain 2,464, 3,850, and 7,700 DAGs in total. The corresponding test datasets consist of 10 problem instances built from unseen DAGs. For the test datasets under the TPC-H-50 and TPC-H-100, TPC-H-150 settings, the DAG structures and the first-type resource demands are identical to those used by Wang et. al. \cite{Wang21nips}; we augment them by introducing a second resource dimension and task-type annotations. Since the TPC-H-30 and TPC-H-200 setting was not considered in the prior work, we generate both the DAGs and the associated resource demands from scratch, following the same generation protocol used by Wang et. al.\cite{Wang21nips}.

We schedule tasks on 3 pools: pool 1 is of type 0, and pools 2 and 3 are of type 1. The first resource capacities of the three pools are $600$, $800$, and $500$, respectively, and the second resource capacities are $260$, $240$, and $240$.  We define the compatibility coefficients as follows: for pool type 0, the coefficients are $(1.0, 1/0.8, 1/0.7)$ for task types 0, 1, and 2, respectively; for pool type 1, the coefficients are $(0, 1/1.0, 1/1.1)$. 

To evaluate the generalization ability of our model, we train it on the original environment and then test its performance across the following eight distinct fluctuating environmental settings.
\begin{itemize}
    \item \textbf{More pools:} (a) adding an extra pool with capacity $(500,\, 240)$ and type 1; (b) adding an extra pool with capacity $(200,\, 100)$ and type 0; (i) adding three copies for each pool.
  
    \item \textbf{More pool types:} (c) modifying the type of pool 3 to a different new pool type with compatibility coefficients $(1/2.0,\, 1/1.0, \, 1/1.1)$; (d) modifying the type of pool 3 to a different new pool type with compatibility coefficients $(0,\, 1/1.2,\, 1/0.8)$.
    \item \textbf{More task types:} (e) reclassifying half of the type 2 tasks as a new task type, which has a compatibility coefficient of $1/1.4$ with pool type 0 and $1/0.8$ with pool type 1; (f) reclassifying half of the type 2 tasks as another new task type, which has a compatibility coefficient of $1/1.0$ with pool type 0 and $1/2.0$ with pool type 1, and reclassifying the rest half as the new type in (e).
    \item \textbf{Different pool capacity:} (g) increasing the first resource capacities for the three pools to $800, 1100, 500$ and second capacities to $340, 320, 240$; (h) decreasing the first resource capacities for the three pools to $600, 800, 500$ and second capacities to $180, 180, 240$.
\end{itemize}
The fluctuation scenarios in Figure 6 correspond to the environments described above.

\subsection{Computation graphs}

We use the three types of computation graphs introduced by Jeon et. al. \cite{jeon2023iclr}. The first type, the layered graph, is a synthetic structure designed to resemble the computational graphs commonly found in certain classes of neural networks. The other two types are established families of random undirected graphs, which we convert into DAGs following the same transformation procedure described by Jeon et. al. \cite{jeon2023iclr}. We generate graphs with specific configurations based on Jeon et al.'s methodology. For the layered graph, we use the parameters $\sigma_N =0.75,  \rho_E=0.2, \rho_S=0.14$. For the Erd\H{o}s-R\'{e}nyi graph, the edge probability is set to $p=0.05$. For the stochastic block model, we use the parameters $p_{in}=0.3$ and $p_{out}=0.005$. For each instance, we generate 10 graphs of size 50, resulting in a total of 500 task nodes per problem. 

Task durations are generated following the approach of Jeon et. al. ~\cite{jeon2023iclr}. The duration $t(v)$ for task $v$ is calculated as:
\begin{equation*}
    t(v) = round(100\times m(v))+1,
\end{equation*}
where $m(v)$ is sampled from a Gaussian mixture model (GMM) and projected onto the non-negative domain. The GMM consists of four Gaussian components, each with equal mixture weight. The means of the components are $(0.5,\,1,\,3,\,5)$, and their corresponding standard deviations are $(0.5,\,1,\,1,\,1)$. For resource requirements, each task's processor demand is uniformly sampled from the set $\{2,\, 4,\, 8,\, 16\}$, and memory demand is selected from $\{1,\, 2,\, 3\}$, also with uniform probability. Task types are sampled from $\{0,\,1,\,2\}$ with probability $1/6,\,1/6,\,2/3$, respectively. Task scheduling is performed with three resource pools: pool 1 is of type 0, while pools 2 and 3 are of type 1.The first resource capacities of the three pools are $16,\, 12,\, 64,$ respectively, and the second resource capacities are $15,\, 20,\, 50$. The compatibility coefficients of pools with type 0 are $(1.0,\,1/0.8,\,1/1.2)$ and those for pools with type 1 are $(0,\,1/1.2,\,1/0.8)$.

\subsection{RI-W Datasets}

We further construct resource-intensive workload (RI-W) datasets to evaluate scheduling performance under workloads whose execution is affected by a small number of resource-intensive tasks. In this setting, most tasks follow the standard TPC-H workload distribution, while a small fraction of tasks require substantially larger resource demands and longer processing times. These resource-intensive tasks may occupy major resource capacities for long periods and create bottleneck decisions involving waiting, packing, and task ordering. Such workloads provide a challenging setting in which restricted generation maps are more likely to exhibit generation-induced suboptimality. We instantiate the RI-W datasets by deriving query instances from the TPC-H dataset and modifying a small fraction of tasks to emulate these resource-intensive jobs.

For the RI-W-30 and RI-W-50 datasets, we start from the corresponding TPC-H-30 and TPC-H-50 settings. We identify non-root tasks with durations exceeding $1000$ seconds, randomly select half of them, and assign them a resource demand of $(300,210)$ and task type 0. Approximately $1.6\%$ of all tasks are replaced in this manner, resulting in an average of $4.8$ and $8.0$ heavy tasks per problem instance in the RI-W-30 and RI-W-50 datasets, respectively. To accommodate these heavy tasks, the first resource capacities of the three pools are increased to $1200$, $3000$, and $800$, while the second resource capacities are set to $260$, $260$, and $240$. The compatibility coefficients are kept unchanged.

To further evaluate the ability of algorithms to address optimality gaps under different workload intensities, we construct additional RI-W variants based on TPC-H-30 and TPC-H-50. Specifically, we identify all non-root DAG tasks and replace $0.4\%$, $0.8\%$, $1.6\%$, and $3.2\%$ of them with heavy tasks that have resource demand $(300,210)$, task type 0, and processing time $2000$. The environment settings and compatibility coefficients are the same as those used in RI-W-30 and RI-W-50. These variants are used to study how the benefit of skip-extended generation changes as the ratio of resource-intensive tasks increases.

\subsection{Problem size of benchmarks}

We summarize the key statistics regarding the average problem size of the benchmarks. Table~\ref{tab:pb_sz} reports the size of the action set, the number of MILP variables, the number of MILP constraints, and the number of nonzero elements in the MILP constraints. The RI-W datasets have the same DAG counts and pool numbers as the corresponding TPC-H-30 and TPC-H-50 settings; they differ in task durations and resource demands rather than in problem-size order. The two homogeneous lines with a single pool correspond to the largest benchmark instances tested in prior works \cite{Wang21nips,jeon2023iclr}. The heterogeneous cases correspond to the benchmarks evaluated in our study. As shown in the table, the heterogeneous settings substantially enlarge the problem size, making the benchmarks we test considerably more challenging than those in previous works.

\begin{table}[htbp]
\centering
\caption{Problem size of Benchmarks}
\label{tab:pb_sz}
\renewcommand{\arraystretch}{1.2}
\begin{tabular}{lcccc}
\toprule
& Action set & Variables & Constraints & Nonzeros\\
\hline
Homogeneous & $n$ & $3n^2$ & $6n^2$ & $19n^2$ \\
\hline
TPC-H-150, 1 pool & 1500 & $6.8 \times 10^6$ & $1.4 \times 10^7$ & $4.3 \times 10^7$ \\
\hline
TPC-H-300, 1 pool & 3000 & $2.7 \times 10^7$ & $5.4 \times 10^7$ & $1.7 \times 10^8$ \\
\hline
Heterogeneous & $nn_c$ & $(3+n_c)n^2$ & $(6+2n_c)n^2$ & $(18+8n_c+rn_c)n^2$ \\
\hline
Computation-Graph, 3 pool & 1500 & $1.5 \times 10^6$ & $3 \times 10^6$ & $1.2 \times 10^7$ \\
\hline
TPC-H-100, 3 pool & 3000 & $6.0 \times 10^6$ & $1.2 \times 10^7$ & $4.8 \times 10^7$ \\
\hline
TPC-H-200, 3 pool & 6000 & $2.4 \times 10^7$ & $4.8 \times 10^7$ & $2 \times 10^8$ \\
\hline
TPC-H-200, 12 pool & 24000 & $6.0 \times 10^7$ & $1.2 \times 10^8$ & $5.5 \times 10^8$ \\
\bottomrule
\end{tabular}
\end{table}

\section{Baseline algorithms}\label{apx:baselines}
\subsection{Heuristic baselines}
We here provide implementation details of heuristic baselines. For clarity, we denote the average processing time $\Bar{t}(v)$ for task $v$ as the average of its actual processing times across all compatible resource pools $c$ where $K_{acc}(v, c) > 0$.
\begin{itemize}
    \item \textbf{Shortest first task scheduling (SFT).} The scheduler calculates the average processing time for each task and selects the unmasked task with the smallest value.
    \item \textbf{Most operations remaining (MOPNR).} For each task $v$, we define the number of remaining operations  as the number of tasks that directly or indirectly depend on $v$. At each step, the scheduler selects the unmasked task with the largest number of remaining operations.
    \item \textbf{Critical path scheduling (CP).} A directed path originating from task $v$ is any path leading from $v$ to one of its descendants $w$ in the DAG, i.e., a sequence $v = v_0, v_1, \dots, v_k = w$ in which $(v_i, v_{i+1}) \in E$ for all $i = 0, \dots, k-1$. The length of a path is defined as the sum of the average processing times of all tasks along it. The critical path of a task $v$ refers to the longest directed path originating from $v$. The scheduler selects the unmasked task with the largest critical path length.
    \item \textbf{Tetris\cite{tetris}.} A heuristic method for multi-pool cases where jobs are arranged as a Tetris game on the two-dimension space of makespan and resource. The Tetris score between task $v$ and pool $c$ is defined to be the normalized inner product between their resource feature. Let $\rho(v)$ be the resource demand vector of task $v$ and $\lambda(c), \lambda_{cur}(c)$ be the full resource vector and the current resource vector of pool $c$. The Tetris score is defined to be $s_{tetris}(v,c) = (\rho(v)/\lambda(c)) \cdot (\lambda_{cur}(c)/\lambda(c)) = \sum_{k=1}^r (\rho^k(v)/\lambda^k(c)) \cdot (\lambda_{cur}^k(c)/\lambda^k(c))$. The scheduler keeps choosing the unmasked task with the highest Tetris score.
    \item \textbf{HEFT\cite{heft}.} A dependency-aware scheduling algorithm designed for heterogeneous DAG scheduling. It maintains a timeline for each pool and inserts the unmasked task with the highest critical path length into the timeline gap where it can complete at the earliest possible time among all feasible pools.
    \item \textbf{PEFT\cite{PEFT}.} A dependency-aware scheduling algorithm designed for heterogeneous DAG scheduling. It is similar to HEFT but takes the average optimal cost table value as the priority.
    \item \textbf{IPPTS\cite{IPPTS}.} A dependency-aware scheduling algorithm designed for heterogeneous DAG scheduling. It is similar to HEFT but takes the multiplication of the average on ``predict cost matrix'' and out degree as the priority.
\end{itemize}

For each list scheduling based heuristic algorithm (i.e., all except HEFT, PEFT, IPPTS), the pool assignment throughout the scheduling process is determined by one of the following three rules:
\begin{itemize}
    \item \textbf{Earliest finishing time (EFT).} For a given task $v$, the scheduler selects the pool $c$ that allows the earliest feasible completion, which, in the context of resource compatibility, corresponds to the pool with the largest $K_{acc}(v, c)$ value.
    \item \textbf{Tetris score.} For a given task $v$, the scheduler selects the compatible pool $c$ such that the Tetris score between $v$ and $c$ is the highest.
    \item \textbf{Balance choosing.} For a given task $v$, the scheduler selects the pool $c$ that achieves a balance between fast execution and good resource fit.  Specifically, it selects the pool with the largest combined score defined as $s_{tetris}(v,c)\cdot K_{acc}(v,c)$.
\end{itemize}
Each list scheduling based heuristic (SFT, MOPNR, CP, Tetris) is evaluated under all three pool assignment rules, resulting in three complete schedules. The smallest makespan is reported as the final result, and run time is the sum obtained across the three runs.

\subsection{Neural baselines} 
We provide the implementation details of neural baselines. PPO-BiHyb and One-Shot employ the earliest finishing time as the rule for pool selection. 

\begin{itemize} 
    \item \textbf{PPO-BiHyb\cite{Wang21nips}.} The network architecture, as well as the training and inference settings, remain consistent with the original source code, except for the node features and the underlying heuristic algorithm. Specifically, the node features include the average processing time (averaged over all pool types) and the real processing times for each individual pool type. We observe that IPPTS (in TPC-H) and HEFT (in Computation graphs) achieve the best performance among all heuristic baselines. Therefore, we adopt them as the underlying heuristic algorithm in the bilevel framework of PPO-BiHyb. 
    \item \textbf{OneShot\cite{jeon2023iclr}.} As the source code of Topoformer is unavailable, Graphormer with additional attention masks is used as the encoder instead. Each task has a feature vector $(\Bar{t}, \rho_1,\rho_2)$ of dimension 3, where $\Bar{t}$ is the average processing time. This design enables OneShot to handle varying environment structures, facilitating the comparison of its generalization ability with that of our WeCAN. The training settings align with those of our WeCAN, except for an additional $\ell_2$ regularization term in the loss function mentioned in their paper \cite{jeon2023iclr}.

    \item \textbf{ScheduleNet\cite{Wang2022SchNet}.}
    We adapt the ScheduleNet-style heterogeneous GAT encoder to our task--pool DAG scheduling by constructing a two-type heterogeneous graph
    (task nodes and pool nodes) and using \emph{four} directed relations:
    (i) task$\to$pool edges,
    (ii) pool$\to$task reverse edges,
    (iii) precedence edges task$\to$task,
    and (iv) reverse precedence edges (treating the DAG direction explicitly).
    The encoder performs relation-specific attention/message passing and type-aware updates (task/pool have separate transformations) and the actual processing time $
t_{\mathrm{act}}(v,c):=t(v)/K_{acc}(v,c)$ is embedded into edge features,
    producing embeddings $\{h_v\}$ and $\{h_c\}$ in a single forward pass.
    We then score feasible task--pool pairs by a shared decoder
    \[
    u(v,c)=\mathrm{Dec}\bigl([h_v \Vert h_c]\bigr).
    \]
    The resulting logits are plugged into the same feasibility-masked generator used in our framework.

    \item \textbf{HGN-Two-Stage\cite{Song2023TwoStage}.}
    The original method targets FJSP with a two-stage heterogeneous encoder.
    We adapt its \emph{two-stage} update pattern to the task--pool DAG setting (operations$\mapsto$tasks, machines$\mapsto$pools)
    while replacing the FJSP neighborhoods with (a) task--pool edges and (b) directed precedence edges.
    Each layer performs:
    \emph{Stage-I (pool update)}: update pool embeddings by attending over incident task$\to$pool edges, augmented with pool self-loops
    (neighbor edges and self-loops are normalized together, but use distinct projections);
    \emph{Stage-II (task update)}: update task embeddings by fusing three streams:
    forward-precedence attention, backward-precedence attention, and a pool$\to$task aggregation (implemented as a mean aggregator followed by a small MLP),
    concatenated with the task residual and passed through an MLP.
    After the final layer, we use the same pairwise decoding $u(v,c)$ and the same feasibility-masked generator as above.
\end{itemize}

\noindent\textbf{Implementation alignment.}

For OneShot, ScheduleNet and HGN-Two-Stage, we match the main embedding dimension and depth to our WeCAN to control parameter budget.
For ScheduleNet and HGN-Two-Stage, we keep only the heterogeneous encoders and replace the original step-wise decision heads by the same single-pass decoding and generation interface used in our framework.
Moreover, the hidden dimension, layer number, and attention-head configuration are set identical to our WeCAN, ensuring matched parameter budgets across neural methods.

\section{Additional results}\label{apx:add_results}

\subsection{Additional results on generalization ability}

We present full experimental results evaluating the generalization ability of WeCAN. In Tables~\ref{table:flua} to~\ref{table:fluh}, we report the performance of models trained in a fixed TPC-H environment and evaluated under eight environment fluctuations (a)-(h), as defined in Appendix~\ref{apx:dataset}. Additionally, in Table~\ref{table:cs} we report the performance of models trained on TPC-H-30 while tested on TPC-H-50 and TPC-H-100, which contain more tasks. In this table, WeCAN-Train30 and OneShot-Train30 indicate the models trained on TPC-H-30 while WeCAN and OneShot denote models trained and tested under the same task setting. The TPC-H-30 results in Tables \ref{table:flua}- \ref{table:fluh} correspond to the fluctuations in Figure 7.

The results show that our WeCAN  consistently outperforms OneShot and heuristic baselines across all evaluation scenarios. Moreover, it maintains stable performance under both fluctuating environments and varying task numbers, demonstrating strong generalization ability.

\begin{table*} [htbp]
\caption{Experimental results on TPC-H datasets on fluctuating environment (a), which corresponds to the ``more resource pool (A)'' in Figure 7.} 
  \label{table:flua}
  \setlength{\tabcolsep}{3.5pt}
  \centering
  \begin{tabular}{lccccccr}
    \toprule
     & \multicolumn{2}{c}{TPC-H-30, 3 pools} & \multicolumn{2}{c}{TPC-H-50, 3 pools} & \multicolumn{2}{c}{TPC-H-100, 3 pools}\\
        & {Makespan} & {Improvement} & {Makespan} & {Improvement}  & {Makespan} & {Improvement}\\
    \midrule
    SFT     & 25556 & -26.15\% & 43289& -19.12\%  & 77494 & -18.64\%\\
    MOPNR & 23080  & -13.93\%   & 40724 & -12.06\% & 71166 & -8.95\%\\
    Tetris & 21889 & -8.05\%   & 37682 & -3.69\% & 67923 &  -3.99\%\\
    CP & 22431 & -10.73\%   & 39041 & -7.43\% & 68278 &  -4.53\%\\
    HEFT & 22622 & -11.67\%   & 38643 & -6.34\% & 67872 &  -3.91\%\\
    PEFT & 22012 & -8.66\%   & 38111 & -4.87\% & 67872 &  -1.44\\
    IPPTS & 20258 & 0.00\%   & 36340 & 0.00\% & 65318 &  0.00\%\\
    \midrule
    OneShot-S(256)    & 19830 &  2.11\% & 62744 & 5.75\%  & 62744 & 3.94\% \\
    \midrule
    WeCAN-Greedy &  18167    & 10.32\%   & 32244 & 11.27\% & 58764 & 10.03\% \\
    WeCAN-S(64)   &  17496  &  13.64\%  & 31416 & 13.55\% & 57732 & 11.61\% \\
    WeCAN-S(256)   &  \textbf{17430}  &  \textbf{13.96\%}  & \textbf{31157} & \textbf{14.26\%} & \textbf{57404} & \textbf{12.12\%} \\
    \bottomrule
  \end{tabular}
\end{table*}

\begin{table*}[htbp]
  \caption{Experimental results on TPC-H datasets on fluctuating environment (b), which corresponds to the ``more resource pool (B)'' in Figure 7.} 
  \label{table:flub}
  \setlength{\tabcolsep}{3.5pt}
  \centering
  \begin{tabular}{lccccccr}
    \toprule
     & \multicolumn{2}{c}{TPC-H-30, 3 pools} & \multicolumn{2}{c}{TPC-H-50, 3 pools} & \multicolumn{2}{c}{TPC-H-100, 3 pools}\\
        & {Makespan} & {Improvement} & {Makespan} & {Improvement}  & {Makespan} & {Improvement}\\
    \midrule
    SFT     & 23741 & -29.89\% & 41476 & -33.97\%  & 74061 & -27.30\%\\
    MOPNR & 21869  & -19.65\%   & 38129 & -23.16\% & 67597 & -16.19\%\\
    Tetris & 20000 & -9.42\%   & 32237 & -4.13\% & 58796 &  -1.06\%\\ 
    CP & 21089 & -15.38\%   & 34598 & -11.75\% & 62874  &  -8.07\%\\
    HEFT & 18624 & -1.89\%   & 31866 & -2.93\% & 58531 &  -0.60\%\\
    PEFT & 19834 & -8.51\%   & 32719 & -5.68\% & 59118 &  -1.61\%\\
    IPPTS & 18278 & 0.00\%   & 30960 & 0.00\% & 58180 &  0.00\%\\
    \midrule
    OneShot-S(256)    & 17330 &  5.19\% & 29807 & 3.72\%  & 55739 & 4.20\% \\
    \midrule
    WeCAN-Greedy &  17119    & 6.34\%   & 29251 & 5.52\% & 55314 & 4.93\% \\
    WeCAN-S(64)   &  16578  &  9.30\%  & 28580 & 7.69\% & 53891 & 7.37\% \\
    WeCAN-S(256)   &  \textbf{16059}  &  \textbf{9.68\%}  & \textbf{28476} & \textbf{8.02\%} & \textbf{53401} & \textbf{8.21\%} \\
    \bottomrule
  \end{tabular}
\end{table*}

\begin{table*}[htbp]
  \caption{Experimental results on TPC-H datasets on fluctuating environment (c), which corresponds to the ``more pool types (A)'' in Figure 7.} 
  \label{table:fluc}
  \setlength{\tabcolsep}{3.5pt}
  \centering
  \begin{tabular}{lccccccr}
    \toprule
     & \multicolumn{2}{c}{TPC-H-30, 3 pools} & \multicolumn{2}{c}{TPC-H-50, 3 pools} & \multicolumn{2}{c}{TPC-H-100, 3 pools}\\
        & {Makespan} & {Improvement} & {Makespan} & {Improvement}  & {Makespan} & {Improvement}\\
    \midrule
    SFT    & 26626 & -28.43\% & 44708 & -25.14\%  & 80460 & -17.13\%\\
    MOPNR & 24792  & -19.58\%   & 41589 & -16.41\% & 75168 & -9.43\%\\
    Tetris & 22829 & -10.12\%   & 36691 & -2.70\% &  68863 &  -0.25\%\\
    CP & 22519 & -8.58\%   &  38745& -8.45\% & 72774 &  -5.94\%\\
    HEFT & 20732 & 0.00\%   & 35727 & 0.00\% & 68996 &  -0.44\%\\
    PEFT & 21848 & -5.38\%   & 36572 & -2.37\% & 68692 &  0.00\%\\
    IPPTS & 20977 & -1.18\%   & 35871 & -0.40\% & 69374 &  -0.99\%\\
    \midrule
    OneShot-S(256)    & 20545 &  0.90\% & 36091 & -1.02\%  & 69410 & -1.04\% \\
    \midrule
    WeCAN-Greedy &   19816  & 4.42\%   & 32893 & 7.93\% & 65760 & 4.27\% \\
    WeCAN-S(64)   &  19449  &  6.19\%  & 32696 & 8.48\% & 65326 & 4.90\% \\
    WeCAN-S(256)   &  \textbf{19340}  &  \textbf{6.71\%}  & \textbf{32675} & \textbf{8.54\%} & \textbf{65199} & \textbf{5.09\%} \\
    \bottomrule
  \end{tabular}
\end{table*}

\begin{table*}[htbp]
  \caption{Experimental results on TPC-H datasets on fluctuating environment (d), which corresponds to the ``more pool types (B)'' in Figure 7.} 
    \label{table:flud}
  \setlength{\tabcolsep}{3.5pt}
  \centering
  \begin{tabular}{lccccccr}
    \toprule
     & \multicolumn{2}{c}{TPC-H-30, 3 pools} & \multicolumn{2}{c}{TPC-H-50, 3 pools} & \multicolumn{2}{c}{TPC-H-100, 3 pools}\\
        & {Makespan} & {Improvement} & {Makespan} & {Improvement}  & {Makespan} & {Improvement}\\
    \midrule
    SFT     & 27303 & -27.61\% & 47400 & -27.93\%  & 84195 & -24.71\%\\
    MOPNR &  24964 & -16.68\%   & 43010 & -16.08\% & 77059 & -14.14\%\\
  
    Tetris & 23203 & -8.45\%   & 38885 & -4.95\% &  71034 &  -5.21\%\\
    CP & 23129 & -8.10\%   & 41465 & -11.91\% & 73481 &  -8.84\%\\
    HEFT & 22416 & -4.77\%   & 38648 & -4.31\% & 69443 &  -2.86\%\\
    PEFT & 22658& -5.90\%   & 39257 & -5.95\% & 68984 &  -2.18\%\\
    IPPTS & 21396 & 0.00\%   & 37052 & 0.00\% & 67515 &  0.00\%\\
    \midrule
    OneShot-S(256)    & 19938 &  6.81\% & 35018 & 5.49\%  & 65561 & 2.89\% \\
    \midrule
    WeCAN-Greedy &   19146  & 10.51\%   & 33282 & 10.18\% & 61271 & 9.25\% \\
    WeCAN-S(64)   &  18716  &  12.52\%  & 32526 & 12.22\% & 60285 & 10.71\% \\
    WeCAN-S(256)   &  \textbf{18569}  &  \textbf{13.21\%}  & \textbf{32417} & \textbf{12.51\%} & \textbf{60051} & \textbf{11.06\%} \\
    \bottomrule
  \end{tabular}
\end{table*}

\begin{table*}[htbp]
  \caption{Experimental results on TPC-H datasets on fluctuating environment (e), which corresponds to the ``more task types (A)'' in Figure 7.} 
    \label{table:flue}
  \setlength{\tabcolsep}{3.5pt}
  \centering
  \begin{tabular}{lccccccr}
    \toprule
     & \multicolumn{2}{c}{TPC-H-30, 3 pools} & \multicolumn{2}{c}{TPC-H-50, 3 pools} & \multicolumn{2}{c}{TPC-H-100, 3 pools}\\
        & {Makespan} & {Improvement} & {Makespan} & {Improvement}  & {Makespan} & {Improvement}\\
    \midrule
    SFT     & 27254 & -26.86\% & 49040 & -27.47\%  & 83922 & -24.88\%\\
    MOPNR &  24809 & -15.47\%   & 43135 & -12.12\% & 77623 & -15.51\%\\
    Tetris & 23629 & -9.98\%   & 38484 & -0.03\% &  71358 &  -6.19\%\\
    CP & 23488 & -9.33\%   & 41564 & -8.04\% & 72609 &  -8.05\%\\
    HEFT & 22768 & -5.97\%   & 40136 & -4.33\% & 68243 &  -1.55\%\\
    PEFT & 22521 & -4.82\%   & 39917 & -3.76\% & 68598 &  -2.08\%\\
    IPPTS & 21484 & 0.00\%   & 38471 & 0.00\% & 67201 &  0.00\%\\
    \midrule
    OneShot-S(256)    & 20435 &  4.88\% & 35093 & 8.78\%  & 64709 & 3.71\% \\
    \midrule
    WeCAN-Greedy &   18908 & 11.99\%   & 33950 & 11.75\% & 61759 & 8.10\% \\
    WeCAN-S(64)   &  18456  &  14.10\%  & 33179 & 13.75\% & 60752 & 9.60\% \\
    WeCAN-S(256)   &  \textbf{18365}  &  \textbf{14.52\%}  & \textbf{32994} & \textbf{14.24\%} & \textbf{60278} & \textbf{10.3\%} \\
    \bottomrule
  \end{tabular}
\end{table*}

\begin{table*}[htbp]
  \caption{Experimental results on TPC-H datasets on fluctuating environment (f), which corresponds to the ``more task types (B)'' in Figure 7.} 
    \label{table:fluf}
  \setlength{\tabcolsep}{3.5pt}
  \centering
  \begin{tabular}{lccccccr}
    \toprule
     & \multicolumn{2}{c}{TPC-H-30, 3 pools} & \multicolumn{2}{c}{TPC-H-50, 3 pools} & \multicolumn{2}{c}{TPC-H-100, 3 pools}\\
        & {Makespan} & {Improvement} & {Makespan} & {Improvement}  & {Makespan} & {Improvement}\\
    \midrule
    SFT     & 29738 & -18.8\% & 50118 & -27.54\%  & 94002 & -21.59\%\\
    MOPNR & 28402  & -13.46\%   & 45473 & -15.72\% & 86306 & -11.63\%\\
    Tetris & 25983 & -3.80\%   & 40322 & -2.61\% &  78031 &  -0.93\%\\
    CP & 26427 & -5.57\%   & 43760 & -11.36\% & 85981 & -11.21\%\\
    HEFT & 26707 & -6.69\%   & 41042 & -4.44\% & 83114 &  -7.50\%\\
    PEFT & 25812 & -3.11\%   & 39556 & -0.66\% & 77313 &  0.00\%\\
    IPPTS & 25032 & 0.00\%   & 39296 & 0.00\% & 77989 &  -0.87\%\\
    \midrule
    OneShot-S(256)    & 22573 &  9.82\% & 36077 & 8.19\%  & 69004 & 8.18\% \\
    \midrule
    WeCAN-Greedy &   22864  & 8.66\%   & 36088 & 8.16\% & 68934 & 10.84\% \\
    WeCAN-S(64)   &  21847  &  12.73\%  & 34539 & 12.11\% & 68299 & 11.66\% \\
    WeCAN-S(256)   &  \textbf{21801}  &  \textbf{12.91\%}  & \textbf{34378} & \textbf{12.52\%} & \textbf{67722} & \textbf{12.41\%} \\
    \bottomrule
  \end{tabular}
\end{table*}

\begin{table*}[htbp]
  \caption{Experimental results on TPC-H datasets on fluctuating environment (g), which corresponds to the ``Larger resource-pool capacity'' in Figure 7.} 
    \label{table:flug}
  \setlength{\tabcolsep}{3.5pt}
  \centering
  \begin{tabular}{lccccccr}
    \toprule
     & \multicolumn{2}{c}{TPC-H-30, 3 pools} & \multicolumn{2}{c}{TPC-H-50, 3 pools} & \multicolumn{2}{c}{TPC-H-100, 3 pools}\\
        & {Makespan} & {Improvement} & {Makespan} & {Improvement}  & {Makespan} & {Improvement}\\
    \midrule
    SFT     & 20842 & -29.98\% & 36076 & -29.92\%  & 63677 & -21.21\%\\
    MOPNR &  18417 & -17.58\%   & 32614 & -17.45\% & 57660 & -9.75\%\\
    Tetris & 18417 & -14.86\%   & 30438 & -9.62\% &  54054 &  -2.89\%\\
    CP & 17790 & -10.95\%   & 29502 & -6.24\% & 53009 &  -0.90\%\\
    HEFT & 17360 & -8.26\%   & 29010 & -4.47\% & 52935 &  -0.76\%\\
    PEFT & 17799 & -11.01\%   & 30286 & -9.07\% & 53773 &  -2.35\%\\
    IPPTS & 16035 & 0.00\%   & 27768 & 0.00\% & 52536 &  0.00\%\\
    \midrule
    OneShot-S(256)    & 15592 &  2.76\% & 27420 & 1.25\%  & 50939 & 3.04\% \\
    \midrule
    WeCAN-Greedy &   15982  & 0.33\%   & 27050 & 2.58\% & 50361 & 4.14\% \\
    WeCAN-S(64)   &  15013  &  6.37\%  & 26159 & 5.80\% & 48882 & 6.96\% \\
    WeCAN-S(256)   &  \textbf{13498}  &  \textbf{6.70\%}  & \textbf{25990} & \textbf{6.40\%} & \textbf{48711} & \textbf{7.28\%} \\
    \bottomrule
  \end{tabular}
\end{table*}

\begin{table*}[htbp]
  \caption{Experimental results on TPC-H datasets on fluctuating environment (h), which corresponds to the ``Smaller resource-pool capacity'' in Figure 7.} 
    \label{table:fluh}
  \setlength{\tabcolsep}{3.5pt}
  \centering
  \begin{tabular}{lccccccr}
    \toprule
     & \multicolumn{2}{c}{TPC-H-30, 3 pools} & \multicolumn{2}{c}{TPC-H-50, 3 pools} & \multicolumn{2}{c}{TPC-H-100, 3 pools}\\
        & {Makespan} & {Improvement} & {Makespan} & {Improvement}  & {Makespan} & {Improvement}\\
    \midrule
    SFT     & 33138 & -28.00\% & 54852 & -18.54\%  & 100653 & -17.54\%\\
    MOPNR & 29829  & -15.22\%   & 50931 & -10.07\% & 92989 & -8.59\%\\
    Tetris & 29179 & -12.71\%   & 48891 & -5.66\% &  91940 &  -7.36\%\\
    CP & 27760 & -7.23\%   & 46985 & -1.54\% &  86498 &  -1.01\%\\
    HEFT & 28119 & -8.61\%   & 47602 & -2.87\% & 88638 &  -3.50\%\\
    PEFT & 28012 & -8.20\%   & 48255 & -4.29\% & 89048 &  -3.98\%\\
    IPPTS & 25889 & 0.00\%   & 46272 & 0.00\% & 85636 &  0.00\%\\
    \midrule
    OneShot-S(256)    & 25349 &  2.08\% & 44122 & 4.65\%  & 82797 & 3.32\% \\
    \midrule
    WeCAN-Greedy &   24525  & 5.27\%   & 41563 & 10.18\% & 79378 & 7.31\% \\
    WeCAN-S(64)   &  23791  & 8.10\%  & 40778 & 11.87\% & 78293 & 8.57\% \\
    WeCAN-S(256)   &  \textbf{23628}  &  \textbf{8.73\%}  & \textbf{40705} & \textbf{12.03\%} & \textbf{78009} & \textbf{8.91\%} \\
    \bottomrule
  \end{tabular}
\end{table*}

\begin{table*}[htbp]
  \caption{Experimental results on TPC-H datasets. OneShot-Train30 and WeCAN-Train30 refer to models trained in TPC-H-30 while OneShot and WeCAN refer to model trained in problems with corresponding size (TPC-H-50 or TPC-H-100)} 
    \label{table:cs}
  \setlength{\tabcolsep}{3.5pt}
  \centering
  \begin{tabular}{lccccr}
    \toprule
     & \multicolumn{2}{c}{TPC-H-50} & \multicolumn{2}{c}{TPC-H-100}\\
        & {Makespan} & {Improvement} & {Makespan} & {Improvement}  \\
    \midrule
    SFT     & 49172 & -27.2\% & 84986 & -21.2\% \\
    MOPNR & 43545  & -12.7\%   & 74364 & -10.3\% \\
    CP & 41597 & -7.62\%   & 74364 & -6.03\% \\
    Tetris & 38654 & 0.00\%   & 71269 & -1.65\% \\
    HEFT & 36326 & -1.71\%   & 70137 & 0.00\% \\
    \midrule
    OneShot-Train30-S(256)    & 36326 &  6.02\% & 66777 & 4.79\% \\
    OneShot-S(256)    & 35561 &  8.00\% & 66173 & 5.65\% \\
    \midrule
    WeCAN-Train30-Greedy &   33654  & 12.9\%   & 63034 & 10.1\% \\
    WeCAN-Greedy &   33428  & 13.5\%   & 62587 & 10.8\% \\
    WeCAN-Train30-S(64)   &  33265   & 13.9\%   & 61894 & 11.8\% \\
    WeCAN-S(64)  &   32927  & 14.8\%   & 61706 & 12.0\% \\
    WeCAN-Train30-S(256)   &  33137   & 14.3\%   & 61812 & 11.9\% \\
    WeCAN-S(256)   &  \textbf{32885}  &  \textbf{14.9\%}  & \textbf{61364} & \textbf{12.5\%}  \\
    \bottomrule
  \end{tabular}
\end{table*}

\clearpage

\subsection{Results on local search}
\label{apx:results-ls}

In Appendix~\ref{apx:role-ls}, we discussed a local-search variant based on the serial schedule generation scheme $S_{\mathrm{SGS}}$ as an alternative way to mitigate generation-induced optimality gaps. Instead of modifying the rollout rule directly, this variant searches neighboring feasible schedule orders and maps the selected order back to a schedule through $S_{\mathrm{SGS}}$. Here we evaluate this alternative on the RI-W datasets, where resource-intensive tasks create bottleneck decisions involving waiting, packing, and task ordering.

The experimental setup follows the RI-W experiments in the main text, i.e., Table~III and Fig.~4(a--b), with dataset details provided in Appendix~\ref{apx:dataset}. Due to the high computational cost of local search, we report results from a single random seed. We evaluate two versions of WeCAN: one with the decreasing-score skip action and one without the skip action. For both versions, we apply local search with 50 and 100 search steps.

As shown in Table~\ref{table:ls}, local search improves both variants, but the improvement is substantially larger when it is applied to WeCAN without the skip action. This observation is consistent with the order-space analysis: without skip, the generation map inherits the restricted reachability of list scheduling, leaving more room for order-space refinement. In contrast, the skip-extended generation map already introduces controlled waiting and enlarges the reachable set, so the marginal gain from local search becomes smaller. However, local search requires repeated calls to the generation map and therefore incurs significantly higher runtime. These results suggest that local search can mitigate generation-induced suboptimality, but the skip action provides a more computationally efficient way to obtain robust performance on resource-intensive workloads.

\begin{table*}[htbp]
  \caption{Experimental results on RI-W datasets.}
    \label{table:ls}
  \centering
  \begin{tabular}{lccccccr}
    \toprule
     & \multicolumn{3}{c}{RI-W-30} & \multicolumn{3}{c}{RI-W-50}\\
        & {Makespan} & {Improvement} & {Time} & {Makespan} & {Improvement} & {Time}  \\
    \midrule

    IPPTS & 25697 & 0.00\% & 0.21   & 38034 & 0.00\% & 0.51  \\
    \midrule
    WeCAN-NoSkip-Greedy &   27138 & -5.61\% & 0.17   & 39457 & -3.74\% &  0.47 \\
    WeCAN-NoSkip-S(256)   &  26159  & -1.80\% &  2.38  & 38485 & -1.19\% & 5.37 \\
    WeCAN-NoSkip-S(256)+LocalSearch(50)   &  24946 & 2.92\%  &  92.80  & 37726 & 0.81\% & 275.57 \\
    WeCAN-NoSkip-S(256)+LocalSearch(100)   &  24793 & 3.52\%  &  192.56  & 37519 & 1.35\% & 526.22 \\
    WeCAN-Greedy &   25176 & 2.03\%  & 0.13   & 37652 & 1.00\% & 0.36 \\
    WeCAN-S(256)   &  24099 & 6.22\%  & 2.43   & 35695 & 6.15\% &  6.12\\
    WeCAN-S(256)+LocalSearch(50)   &  23490  & 8.59\% &  87.87  & 35321 & 7.13\% & 240.68 \\
    WeCAN-S(256)+LocalSearch(100)   &  \textbf{23482}  & \textbf{8.62\%} &  181.34  & \textbf{35300} & \textbf{7.19\%} & 465.11 \\

    \bottomrule
  \end{tabular}
\end{table*}

\subsection{Details on inference time}\label{apx:inference_time}

In Table \ref{table:eva-time} we provide a detailed breakdown of the inference time. We observe that network computation accounts for less than $10\%$ of the total inference time, while the generation map dominates the overall runtime. This explains why the total inference time cannot be significantly lower than heuristic baselines, despite using a fast network. Prior methods that rely on multi-round network inference scaling with the number of tasks incur substantially higher inference times than those dictated by the generation map alone. In contrast, our approach requires only a single forward pass of the network, enabling substantially faster inference overall. Additionally, we report inference time when using only CPU (i.e., without GPU acceleration). As shown in Table \ref{table:eva-time}, WeCAN-Greedy on CPU achieves evaluation times only slightly higher than on GPU, demonstrating that our method remains practical and efficient in scheduling scenarios where only CPU resources are available.
\begin{table*}[htbp]
  \caption{Experimental results about inference time (s).} 
  \label{table:eva-time}
  \centering
  \begin{tabular}{lcccccccr}
    \toprule
      & \multicolumn{3}{c}{GPU-inference}& \multicolumn{3}{c}{CPU-inference} \\
     Experiment & \small{Total}  & \small{Generation map} & \small{Network}  & \small{Total} & \small{Generation map} & \small{Network}  \\
    \midrule
    TPC-H-30   & 0.154 & 0.140 & 0.014 & 0.201 & 0.116 & 0.085 \\ 
    TPC-H-50   & 0.480 & 0.461 & 0.019 & 0.587 & 0.442 &  0.145 \\ 
    TPC-H-100   & 1.72 & 1.71 & 0.017 & 1.92 & 1.53 & 0.391 \\ 
    Erd\H{o}s-R\'{e}nyi graph  & 0.533 & 0.512 & 0.020 & 0.548 & 0.393 & 0.155 \\ 
    Layer Graphs  & 0.256 & 0.238 & 0.018 & 0.372 & 0.214 & 0.158 \\ 
    Stochastic Block  & 0.361 & 0.342 & 0.019 & 0.527 & 0.367 & 0.160 \\ 
    \bottomrule
  \end{tabular}
\end{table*}

\section{Implementation details} \label{apx:implementation_details}

\subsection{Model architecture Details of LDDGNN}

In DAG scheduling, edges $e = (v_1,v_2) \in E$ coexisting with another directed path from $l$ from $v_1$ to $v_2$ are redundant; removing them preserves scheduling constraints. This equivalence allows us to focus on the eliminated graph formed by removing all these redundant edges. We defined the lognest directed distance (LDD) $d_e(v,w)$ as the signed length of the longest directed path between $v$ and $w$. The edges of the eliminated graph can be identified through LDD $d_e(v,w)=\pm 1$, enabling using LDD to implicitly represent this structure without explicit construction. Furthermore, to capture broader connectivity, we extend LDD to $d_{e}(v,w) = +\infty$ for undirected connected but not directed connected node pairs and $d_{e}(v,w) = -\infty$ for disconnected node pairs. The LDD is leveraged to define both the attention masks and attention biases within the MHA sub-layer. Eight distinct types of attention mask  selectively perform attention between node pairs based on LDD values.

The LDDGNN leverages the longest directed distances (LDD) within the DAG. Its node embedding update is defined as:
\begin{equation*}
\begin{aligned}
\bm{q}_v^{l,j} &= \bm{W}_{l,j}^Q \bm{h}_v^{l-1}, \quad
\bm{K}^{l,j} = \bm{W}_{l,j}^K [\bm{h}_{v_1}^{l-1},\ldots,\bm{h}_{v_n}^{l-1}], \quad
\bm{V}^{l,j} = \bm{W}_{l,j}^V [\bm{h}_{v_1}^{l-1},\ldots,\bm{h}_{v_n}^{l-1}], \\
\widehat{\bm{h}}_v^{l} 
&= \bm{h}_v^{l-1} + \text{concat}_j \Bigl[
\text{softmax}\Bigl(
\frac{(\bm{q}_v^{l,j})^{\mathsf T}\bm{K}^{l,j}}{\sqrt{d}}
+ [b_{d_e(v,v_1)},\ldots,b_{d_e(v,v_n)}]
+ [M_{v,v_1}^j,\ldots,M_{v,v_n}^j]
\Bigr)\,\bm{V}^{l,j}
\Bigr], \\
\bm{h}_v^{l} &= \widehat{\bm{h}}_v^{l} + \text{MLP}_{(l)}(\widehat{\bm{h}}_v^{l}).
\end{aligned}
\end{equation*}
where $V=\{v_1,...,v_n\}$ is the set of nodes. 

The LDDGNN incorporates attention masks $M_{v,w}^j$ and bias terms $b_{d_e(v,w)}$ both derived from the LDD. Let $N_1=\mathbb Z \cup \{-\infty,+\infty\}, N_2 = \{1\}, N_3 = \{-1\}, N_4 = \{2\}, N_5 = \{-2\}, N_6 = [3,+\infty)\cap \mathbb Z, N_7 = (-\infty,-3]\cap \mathbb Z, N_8 = \{+\infty\}$. Given the number of attention heads, $n_{head}$ (a multiple of 8), the mask $M_{v,w}^j$ ($j=0,...,n_{head}-1$) for each attention head is calculated by
\begin{equation*}
    M_{v,w}^j =\begin{cases}
        0 & \text{if } d_e(v,w) \in N_{\left\lfloor 8j/n_{head} \right\rfloor+1}\\
        -\infty & \text{otherwise}\\
    \end{cases}.
\end{equation*}
For the bias calculation, a hyperparameter $D_{max}=500$ is employed. First, the folded longest directed distance $d_f(v,w)$ is calculated:
\begin{equation*}
    d_f(v,w) =\begin{cases}
        d_e(v,w) & \text{if } d_e(v,w)\in [-D_{max}+1,D_{max}-1]\\
        D_{max}-1 & \text{if } d_e(v,w)\in [D_{max},+\infty)\\
        -D_{max}+1 & \text{if } d_e(v,w)\in (-\infty,-D_{max}]\\
         D_{max} & \text{if } d_e(v,w)= +\infty\\
        -D_{max} & \text{if } d_e(v,w)= -\infty\\
    \end{cases}.
\end{equation*}
Then the bias $b_{d_e(v,w)}$ is obtained by selecting the $(d_f(v,w)+D_{max}+1)$-th parameter from a learnable embedding table of size $2D_{max}+1$. This mechanism is equivalent to applying a one-hot encoding to the folded distance and then passing it through a linear layer.

\subsection{Model hyperparameters}

The hyperparameter configuration for our WeCAN architecture is detailed below. An MLP is used to embed the initial features of tasks, and a separate MLP is used for resource pools. Specifically, each MLP consists of a linear layer that maps from 3-dimensional (for tasks) or 2-dimensional (for pools) input features to the shared embedding dimension $d_{\text{high}} = 512$, followed by a GELU activation. The first weighted cross attention layer and the longest directed distance graph neural network operate within the $d_{\text{high}} = 512$ dimension. Subsequent alternating update layers in the encoder and the entire decoder are projected into the $d_{\text{low}}=128$ dimension to enhance computational efficiency. All weighted cross attention layers employ 8 attention heads. The longest directed distance graph neural network utilizes 8 layers and 16 attention heads where each mask type corresponds to 2 heads. The skip parameters $\alpha$, $\beta$, and $\gamma$ are computed by an MLP consisting of two hidden layers with a hidden dimension of $64$, using GELU and Sigmoid as activation functions.
  
\subsection{Training details}

We train our model using the Adam optimizer with a learning rate of $10^{-4}$. The training batch size is set to 64 by default. For TPC-H-100, a smaller batch size of 32 is adopted due to GPU memory limitations. The model is trained for 800 batches. The training process does not include any information about the test dataset. Training durations for different configurations are summarized in Table \ref{table:train-time}.

We consider two baseline strategies for REINFORCE: (1) the average reward of all samples; (2) the reward of a greedy rollout obtained using a separate model that is periodically synchronized with the current policy. In practice, we find that the average reward baseline leads to faster and smoother convergence. Based on this observation, we adopt it as the default baseline for training.

As detailed in Appendix~\ref{apx:inference_time}, the generation map module is the primary computational bottleneck during both training and inference. To improve efficiency, we implement a PyTorch-based version of the generation map. This GPU-based implementation is used during training and also in sample-based inference, where large batch sizes or many samples are processed. For greedy inference, however, we use a CPU-based version of the generation map, as GPU acceleration provides no runtime benefit in this setting.
\begin{table}[!htbp]
  \caption{Experimental results about training time of WeCAN-Greedy.} 
  \label{table:train-time}
  \centering
  \begin{tabular}{lccr}
    \toprule
     Experiment & Total time & Time per batch \\
    \midrule
    TPC-H-30   & 6.6h & 0.5min \\ 
    TPC-H-50   & 14.5h & 1.1min \\ 
    TPC-H-100   & 26h & 2min\\ 
    Erd\H{o}s-R\'{e}nyi graph & 19h & 1.4 min\\
    Layer Graphs  & 15.5h & 1.2 min\\
    Stochastic Block  & 13.3h & 1 min\\
    RI-W-30   & 6.5h & 0.5min \\ 
    RI-W-50   & 14h & 1.1min \\
    \bottomrule
  \end{tabular}
\end{table}

\subsection{Inference details}

 Five random seeds: $2000,\,2001,\,2002,\,2003,\,2004$ are used for calculating the mean and standard deviation of results in Tables I to VI. A single fixed random seed of $2000$ is used for the results in Tables \ref{table:flua} to \ref{table:ls}.

\end{document}